\documentclass{svproc}
\usepackage{url}

\usepackage[utf8]{inputenc}
\usepackage{graphicx}
\usepackage{amsmath}
\usepackage{amssymb,xcolor}
\usepackage{mathtools}
\usepackage{enumitem}
\usepackage{makecell}
\usepackage{algorithm}
\usepackage{hyperref}
\usepackage{booktabs}
\usepackage{multirow}
\usepackage{upgreek}
\usepackage[scriptsize]{subfigure}

\def\markatright#1{\leavevmode\unskip\nobreak\quad\hspace*{\fill}{#1}}

\DeclareMathOperator*{\argmin}{arg\,min} 

\newcommand{\ql}{Q-learning}

\def\Re{{\mathbb{R}}}
\def\Nat{{\mathbb{N}}}
\def\plan{{\pi}}
\def\P{{\cal P}}

\def\aqval{\hat{Q}^*}

\begin{document}
\mainmatter              

\title{Relating Reinforcement Learning to \\Dynamic Programming-Based Planning} 

\titlerunning{Relating RL to Planning}  
\authorrunning{Georgiev, Timperi, Sakçak, LaValle}
%
\author{Filip V. Georgiev\inst{1}, Kalle G. Timperi\inst{1}, Başak Sakçak\inst{1,2}, Steven M. LaValle\inst{1}%
\thanks{This work was supported by a European Research Council Advanced Grant (ERC AdG, ILLUSIVE: Foundations of Perception Engineering, 101020977, Academy of Finland BANG! 363637).  {\tt\small (e-mail: firstname.lastname@oulu.fi).}}%
}
\institute{Center for Applied Computing,
Faculty of Information Technology \\and Electrical Engineering, 
University of Oulu, Finland \and 
Dept. of Advanced Computing Sciences, Maastricht University, the Netherlands}

\maketitle              

\begin{abstract}
This paper bridges some of the gap between optimal planning and reinforcement learning (RL), both of which share roots in dynamic programming (DP) applied to sequential decision making or optimal control.  Whereas planning typically favors deterministic models, goal termination, and cost minimization, RL tends to favor stochastic models, infinite-horizon discounting, and reward maximization in addition to learning-related parameters such as the learning rate and greediness factor.  A derandomized version of RL is developed, analyzed, and implemented to yield performance comparisons with value iteration and Dijkstra's algorithm using simple planning models.  Next, mathematical analysis shows: 1) conditions under which cost minimization and reward maximization are equivalent, 2) conditions for equivalence of single-shot goal termination and infinite-horizon episodic learning, and 3) conditions under which discounting causes goal achievement to fail.  The paper then advocates for defining and optimizing {\em truecost}, rather than inserting arbitrary parameters to guide operations.  Performance studies are  
extended to the stochastic case, using planning-oriented criteria and comparing value iteration to RL with learning rates and greediness factors.

\end{abstract}

\section{Introduction}

Bellman introduced the principle of DP in the 1950s as a way to efficiently compute solutions to sequential decision-making problems \cite{Bel57}.  The core idea is the {\em principle of optimality}, a simple observation that optimal solutions must be composed of optimal parts, thereby leading to a powerful recurrence relation or differential equation that underlies the Hamilton-Jacobi-Bellman (HJB) equation in continuous-time optimal control, value and policy iteration methods for discrete-time problems, and is useful for algorithm design and complexity analysis well beyond sequential decision making.  In the context of motion planning, Dijkstra's algorithm, A$^*$-search, D$^*$ \cite{KoeLik02,Ste94}, and many others compute optimal plans over grids or more general graphs.  Dijkstra-like planning algorithms, built upon value iteration, compute approximately optimal feedback plans over continuous spaces using sampling and interpolation of value functions \cite{YerLav11}.
All of these methods extend gracefully to settings with either nondeterministic (possibilistic) or stochastic (probabilistic) prediction uncertainty \cite{Lav06}.

In the 1990s, Barto and Sutton introduced a powerful paradigm that cleverly interleaves the discovery of the effect of actions with the calculations of optimal values and plans \cite{SutBar98}.  In preceding works, experimental data was first used to learn or estimate a model of the form $x'=f(x,u)$ for states $x,x'$ and action $u$ (or $p(x'\;|\;x,u)$ in a stochastic setting), and then the model was used as a module in DP methods.  The new paradigm, called {\em reinforcement learning (RL)}, retains all of the structure of Bellman's original DP equations, while inserting learning or discovery directly into the algorithms for computing optimal plans (or policies).  Thus, RL is a fusion of optimal sequential decision-making (control, planning) with machine learning (system estimation, identification).

In recent years, RL has become ubiquitous as an approach to problems in robotics and embodied AI \cite{KobBagPet13,TanAbbHuChaMarSto25}.  As it is applied to motion planning, it has been  
useful as a way to generate motion/control primitives to complement planners (e.g., \cite{Chi19}); however, the relationships across the Bellman-inspired family of methods have become increasingly opaque.  In non-RL works, a clear cost model is formulated in terms of engineering resources, such as time or energy expended, and plans are selected to globally optimize cumulative cost.  In standard RL, a reward model is defined that is biologically inspired and is often shaped or tuned to yield desired outcomes \cite{NgHarRus99}.  Furthermore, RL tends to be formulated as an infinite-horizon problem in which the most preferred way to make cumulative rewards finite is via an arbitrary discount factor that has little meaning in an engineering context (notable exceptions exist \cite{Nai19}).  In non-RL approaches to planning, the episode usually ends when goals are achieved; even though the number of steps may be unbounded, it is finite and can be efficiently handled by termination actions \cite{Lav06} or terminal states \cite{BerTsi96,Put94}, thereby preserving the true cost model.  Finally, RL is formulated exclusively in a stochastic setting, whereas for non-RL methods there is a simple and clear relationship between corresponding algorithms in deterministic and stochastic settings.  Probabilistic modeling is not the `only' or `most general' way to handle uncertainties (e.g., \cite{Kal94}).

To address these concerns,  
we investigate a series of methods along the spectrum from deterministic planning to stochastic RL (to Q-learning~\cite{WatDay92}). Relationships are established through both theoretical analysis and experimental studies. To this end, this paper is closely related to the previous attempts on unifying RL and DP concepts \cite{Ber24,Ber19}. However, in this work, we focus on the interplay between the cost/reward formulations, analyze the effects of randomization and other practices used in RL on the solution through its relation to DP.  
Section \ref{sec:det} starts with optimal deterministic planning and develops a comparable derandomized analog of RL that clarifies the issues of whether a model is given and whether one can jump at will between arbitrary states.  Computation and convergence rates are compared over several examples.  Section \ref{sec:cost} bridges the gap between typical planning and RL formulations through rigorous mathematical analysis of their relationship, with an emphasis on using physically motivated costs in models, as opposed to arbitrary discounts and rewards.  Section \ref{sec:stoch} extends the studies to stochastic state-transition models, which are prevalent for RL.  Section \ref{sec:con} summarizes the main conclusions and remaining challenges.

\section{Optimal Planning and Deterministic RL}\label{sec:det}

\subsection{Optimal planning concepts}

Consider a standard optimal planning problem expressed as a six-tuple 
$\P = (X, U, f, x_1, X_G, \ell)$, in which $X$ is a finite {\em state space},  $U$ is a finite {\em action space}, $f: X \times U \rightarrow X$ is the {\em state transition function}, $x_1 \in X$ is the {\em initial state}, $X_G \subset X$ is the {\em goal set}, $\ell: X \times U \rightarrow [0,\infty)$ is the {\em cost function}.  Actually, $f$ is allowed to be a partial function, in which $U(x) \subseteq U$ denotes the actions available from each $x \in X$. 
A {\em feedback plan, or policy} is a mapping $\pi: X \rightarrow U$, and we use the shorthand $f_\pi(x) := f(x, \pi(x))$ to indicate the next state obtained by applying policy $\pi$ at state $x$.
A policy $\pi$ is evaluated from any $x_1 \in X$ using the stage-additive {\em cost functional}
\begin{equation}\label{eqn:cf}
L(\plan,x_1) = \sum_{k=1}^\infty \ell(x_k,\pi(x_k)),
\end{equation}
in which $k \in \Nat$ indexes over {\em stages}, 
and 
$x_{k+1} = f_\pi(x_k)$. Note that different cost functionals are possible and will be introduced later. 
To take $X_G$ into account, a special {\em termination action} $u_T \in U(x_g)$ exists from every $x_g \in X_G$, and when applied at stage $k$, then $x_{i+1} = x_i = x_g$ for all all $i \geq k$ and $l(x_g,u_T) = 0$.  This ensures that (\ref{eqn:cf}) is finite if the goal is reachable from $x_1$, in spite of having an infinite horizon, that is, infinite number of stages.  A plan or policy $\pi^*$ is {\em optimal} if for all other $\pi$ and all $x_1 \in X$, $L(\plan^*,x_1) \leq L(\plan,x_1)$.  Let $G_\pi : X \rightarrow [0,\infty)$ be the {\em cost-to-go} function for a plan $\pi$, in which each $G_\pi(x)$ is the total cost obtained in (\ref{eqn:cf}) by applying $\pi(x)$ from state $x_1 = x$.  Let $G^* = G_{\pi^*}$ be the {\em optimal cost to go}.  Determining $G^*$ and $\pi^*$ are closely related because:
\begin{equation}\label{eqn:umin}
\pi^*(x) \in \argmin_{u \in U(x)} \{ \ell(x,u) + G^*(x')\},
\end{equation}
in which $x' = f(x,u)$ (see \cite{Lav06}, Ch. 2).  In the case of classical planning, the model is known; thus, $\pi^*$ and $G^*$ can be computed using value iteration or by Dijkstra's algorithm, in which the latter can be seen an efficient version of value iteration. Whereas Dijkstra's algorithm is typically faster than value iteration, value iteration
generalizes more easily to stochastic planning and RL.  
Value iteration proceeds by successively computing cost-to-go functions for each stage, proceeding backwards through time, until the values stabilize (they do not change).  For each $x_k \in X$, the value is calculated as
\begin{equation}\label{eqn:ctgk2}
G^*_k(x_k) =
\min_{u_k \in U(x_k)} \Big\{ \ell(x_k,u_k) +
G^*_{k+1}(x_{k+1}) \Big\} ,
\end{equation}
in which $x_{k+1} = f(x_k,u_k)$.  Most RL algorithms essentially create an asynchronous, step-by-step version of (\ref{eqn:ctgk2}), in which each step processes a state-action pair $(x_k,u_k)$. 

\subsection{Handling imperfect models by derandomizing RL}
\label{sec:derand_rl}

Whereas most RL algorithms are designed considering imperfect models, planning algorithms typically assume a complete model. Therefore, our first task is to carefully relate and compare these.  Two critical options emerge: 1) Is the model given to the algorithm?  Suppose that in the {\em model-based} case, $\P$ is completely given to the algorithm.  In the {\em model-free} case, $\P$ is known `to the gods' but not the algorithm.  The robot can only try actions and then sense the next state and incurred cost.  2) Can the robot freely `jump' between states for planning purposes?  {\em Virtual jumping} occurs in planning by using the model: The algorithm deduces what would happen if action $u$ were applied from state $x$.  By default, we focus on {\em physical jumping}, in which the robot is moved to a new state.

There are three cases based on the critical options:  1) {\bf Model-based.}  If the model is available, then, virtual jumping can be used, allowing algorithms such as Dijkstra or value iteration to solve the optimal planning problem before the robot is ever moved.  2) {\bf Model-free with jumping}.  If physical jumping is free and $X$ were known, then the model can be completely discovered by trying every $u \in U(x)$ from every $x \in X$.  Suppose that even $X$ is unknown.  In this case, we assume sensing is sufficient so that states can be uniquely labeled as they are discovered, and returns to previous states are correctly determined.  Even in this extreme scenario, Dijkstra or value iteration could be applied to derive the optimal plan. The difference compared to the first case is that here, jumping must be physical. (If a cost is assigned for physical jumping, then an intriguing new optimization problem emerges.)  3) {\bf Model-free without jumping.}  This case most closely matches common RL formulations.  The robot must move along one long path to discover an optimal plan.  (The directed transition graph underlying $\P$ is required to have every state reachable from every other.)  

Note that since the system is deterministic, every state-action pair needs to be considered only once to completely discover $f$; however, revisits may be needed to fully compute $G^*$.  Furthermore, $f$ can be represented as a directed multigraph with vertices $X$ and labeled edges $U(x)$ for all $x \in X$.  This leads to a simple two-phase algorithm: 1) First, `physically' explore enough to discover the entire graph (jumping not allowed!); several efficient algorithms exist \cite{AlbHen00,DudJenMilWil93}.  2) Then, run Dijkstra's algorithm or value iteration to calculate $G^*$ and $\pi^*$.
We call the resulting algorithms {\em Model-free Dijkstra} and {\em Model-free Value Iteration}, respectively.
This, however, does not follow the spirit of most RL, which interleaves learning and optimal planning during the entire execution.  To move closer to this approach, we introduce a derandomized version of Q-learning, which serves as a gateway between typical planning and RL scenarios.

Let $Q_\pi : X \times U \rightarrow \Re$ be the {\em Q-value function}, obtained from (\ref{eqn:cf}) by applying $u_1$ at $x_1$, and then applying actions $u_i = \pi(x_i)$ for all $i > 1$.  Let $Q^* = Q_{\pi^*}$ be the {\em optimal Q-value function}.  Note that $G^*$  
can be computed from $Q^*$ by
\begin{equation}\label{eqn:qval2}
G^*(x) = \min_{u \in U(x)} \Big\{ Q^*(x,u) \Big\} .
\end{equation}
Q-learning estimates $Q^*$ through the process of stochastic iteration \cite{SutBar98,RobMon51}. Using a stochastic state transition model $p(x_{k+1}|x_k,u_k)$ (addressed in Section~\ref{sec:stoch}), the estimate $\hat{Q}^*$ of $Q^*$, is determined iteratively by
\begin{equation}\label{eqn:qvalit}
\vspace{-0.25em}
\aqval(x,u) \leftarrow (1-\rho) \aqval(x,u) + \rho \left(\ell(x,u) + \min_{u' \in
U(x')} \Big\{ \aqval(x',u') \Big\}\right) ,
\vspace{-0.25em}
\end{equation}
in which $x'$ is obtained by applying $u$ from $x$ in a single step, $\rho \in (0,1)$ is the {\em learning rate}, and $\leftarrow$ denotes an assignment (see \cite{Lav06}, Sec. 10.4).  Note that jumping is not required: The minimization over all $u' \in U(x')$ involves looking up internally stored Q-values.    The parameter $\rho$ is chosen to be smaller if the amount of uncertainty or instability is higher.  If every reachable state-action pair is visited infinitely often, then $\aqval$ converges to $Q^*$ in probability \cite{Ber19,WatDay92}.  

In the limiting case of a deterministic system, there is no uncertainty, which would suggest setting $\rho = 1$ to obtain our proposed {\em derandomized Q-learning}:
\begin{equation}\label{eqn:qvalitdet}
\vspace{-0.25em}
\aqval(x,u) \leftarrow \ell(x,u) + \min_{u' \in
U(x')} \Big\{ \aqval(x',u') \Big\} .
\vspace{-0.25em}
\end{equation}
Q-learning is considered {\em off-policy} in the sense that no particular plan of movement for exploration is assumed, as long as every state-action pair is visited infinitely often.  In practice, however, a movement plan is needed.  To ensure all state-action pairs are visited infinitely often, two {\em pure exploration} plans are considered: 1) apply random actions, leading to probabilistic convergence, and 2) apply a universal plan from \cite{TimLavLav25} to obtain deterministic convergence.  
\begin{proposition}\label{prop:cosrew}
If every state-action pair $(x,u)$ is visited infinitely often for all $x \in X$ and $u \in U(x)$, and (\ref{eqn:qvalitdet}) is applied in every step, then after a finite number of iterations, $\aqval = Q^*$.
\end{proposition}
\begin{proof}
Let $Z = X\times U$ 
be an augmented state space 
and let  
$f_Z:Z \times U \rightarrow Z$
be the state transition function defined as $f_Z((x,u),u')=(f(x,u),u')$, in which $f$ is the state transition function for the original problem. Rewriting \eqref{eqn:qvalitdet} in $Z$ yields 
$\aqval(z) \leftarrow \ell(z) + \min_{u' \in U(f(x,u))}\{ \aqval(z') \}$,   
in which $z'=f_Z(z,u')$. Because in Q-learning the values are updated when $x$ is visited and $u$ is applied, this can be seen as a step in an asynchronous value iteration with single state updates in contrast with value iteration which does a complete sweep of all the states. Then, the finite time convergence to the optimal values follows from the finite time convergence of asynchronous value iteration for deterministic shortest-path problems \cite{Ber83}. \qed
\end{proof}

As is common in RL, we balance exploration vs. exploitation in the form of an {\em $\epsilon$-greedy} plan/policy.  For some fixed value $\epsilon \in [0,1]$, the policy applies the next action in the pure exploration plan with probability $\epsilon$; otherwise, it applies the minimizing $u'$ in (\ref{eqn:qvalitdet}).  Thus, it is `greedy' (exploitation) $1-\epsilon$ of the time on average. Using $\epsilon$-decay was tested, but excluded; see the appendix. 

We follow another common RL practice, which is to allow periodic jumping back to the initial state.  Thus, the execution is divided into {\em episodes}, in which each one corresponds to applying an action sequence from the initial state until either $X_G$ or a maximum iteration limit is reached (to escape greediness-induced traps).  This issue does not arise in value iteration because the state becomes frozen once $X_G$ is reached and $u_T$ is applied. The relationship between these two settings is analyzed in Section \ref{sec:episode}.


\subsection{Computational comparisons} \label{sec:det_comp_comparisons}

We compared Q-learning and the model-free versions of Dijkstra's algorithm, value iteration, and asynchronous value iteration (see Sec.~\ref{sec:derand_rl}),
over several planning problems encoded as grids with standard 4-neighbors.  Higher dimensions and continuous problems are avoided here because our aim is to improve understanding and we believe the key issues identified in our studies extend to higher dimensional lattices \cite{LavBraLin04} and continuous problems involving DP with interpolation \cite{LavKon01}.
The experiments were conducted using Python\footnote{The codebase can be found \href{https://github.com/filkata123/valit_q_comparison}{here}.}. They were run on a desktop computer with a Ryzen 7 3700X CPU (3.6GHz)\footnote{A computer with an  i5-12400F CPU (2.50GHz) was used for some of the stochastic experiments ($\gamma \leq 0.9$).}.
Figure~\ref{fig:problems} shows a representative sample of the problems we investigated; many 
more results appear in the appendix.

\begin{figure}[h!]
\vspace{-1.5em}
    \centering
        \includegraphics[width=0.18\linewidth]{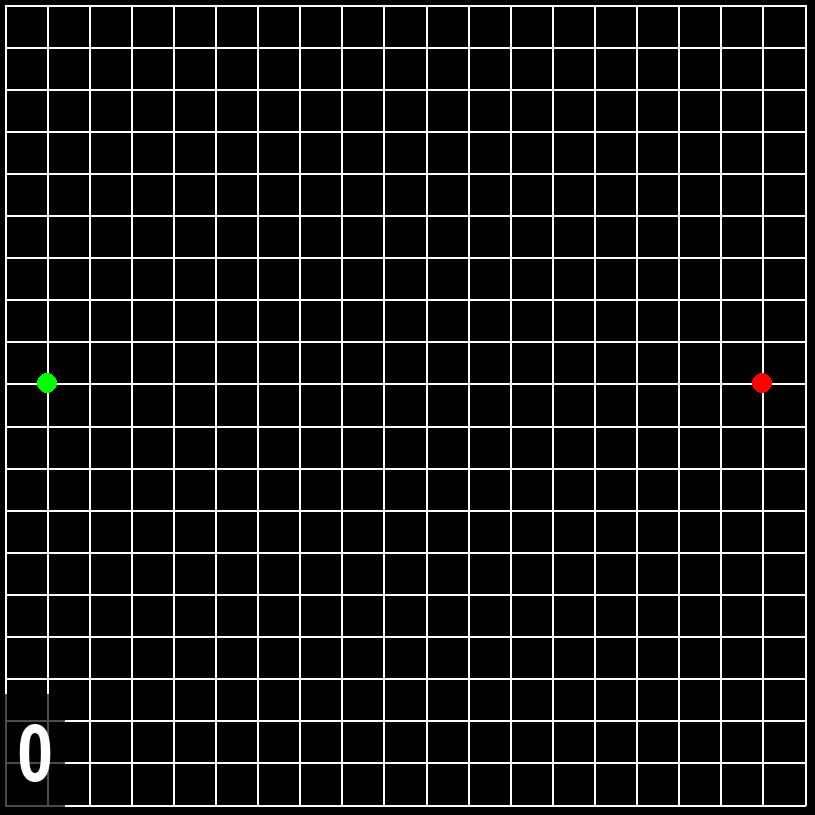}
        \includegraphics[width=0.18\linewidth]{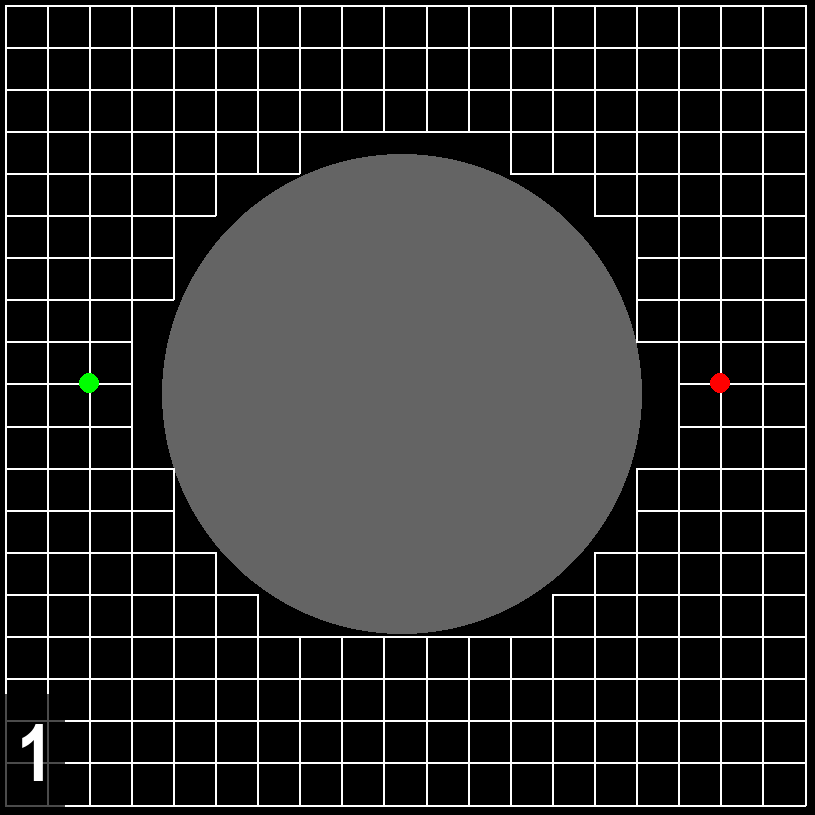}
        \includegraphics[width=0.18\linewidth]{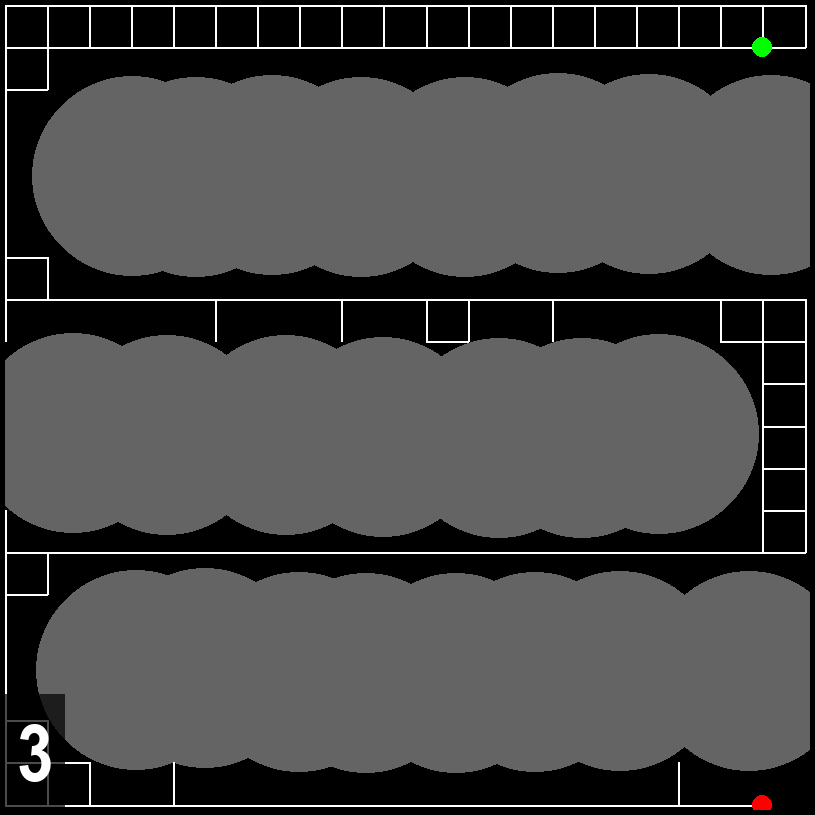}
        \includegraphics[width=0.18\linewidth]{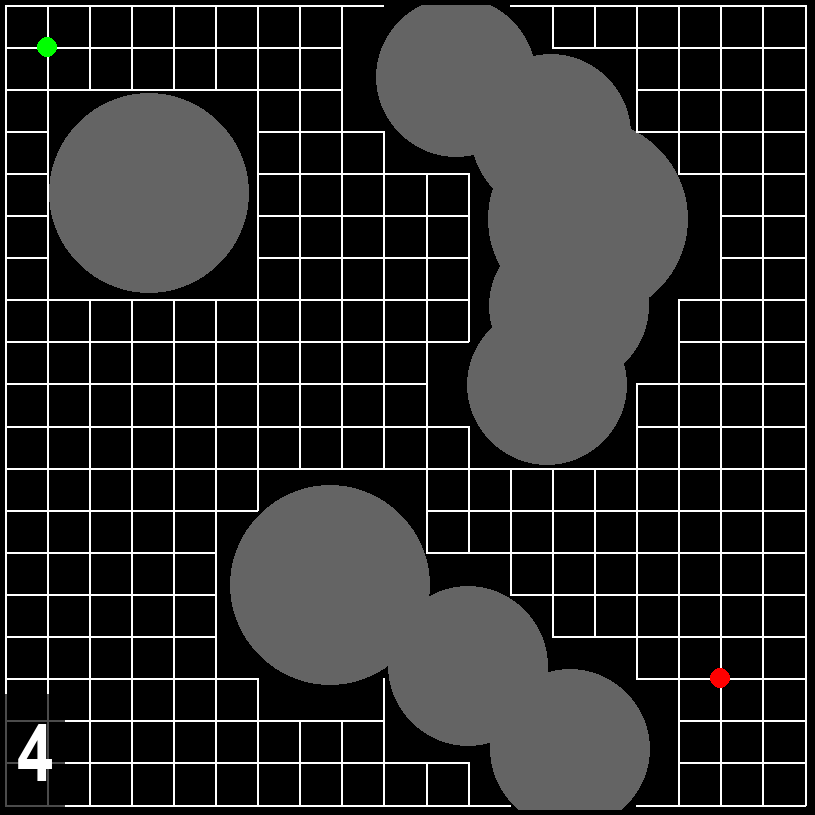}
        \includegraphics[width=0.18\linewidth]{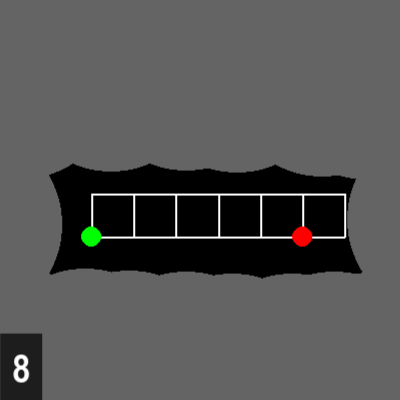}
    \vspace{0.1cm}
    
        \includegraphics[width=0.18\linewidth]{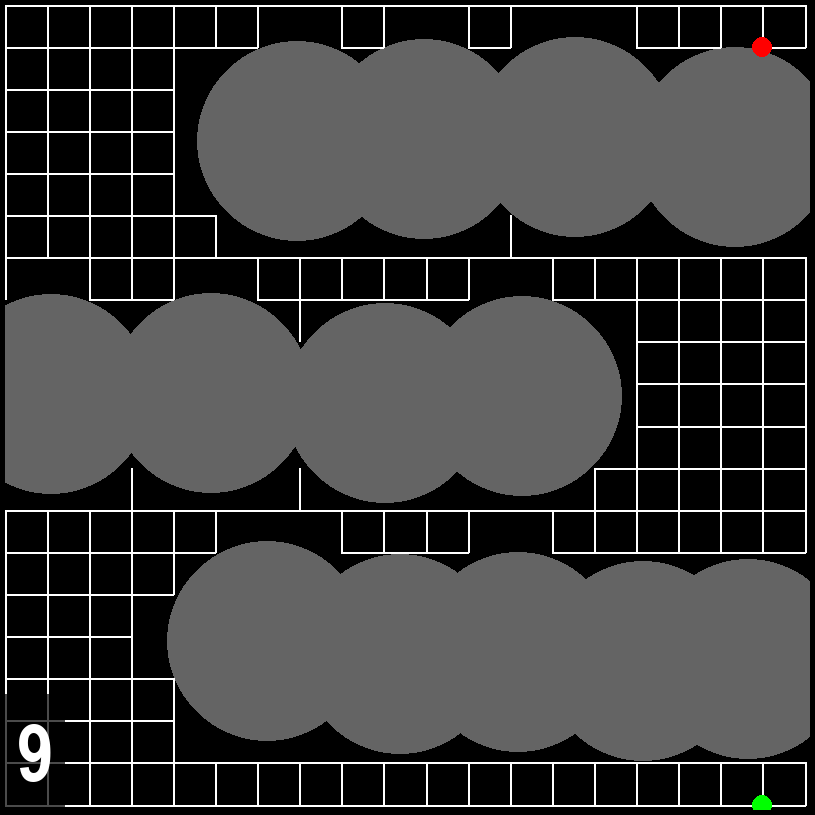}
        \includegraphics[width=0.18\linewidth]{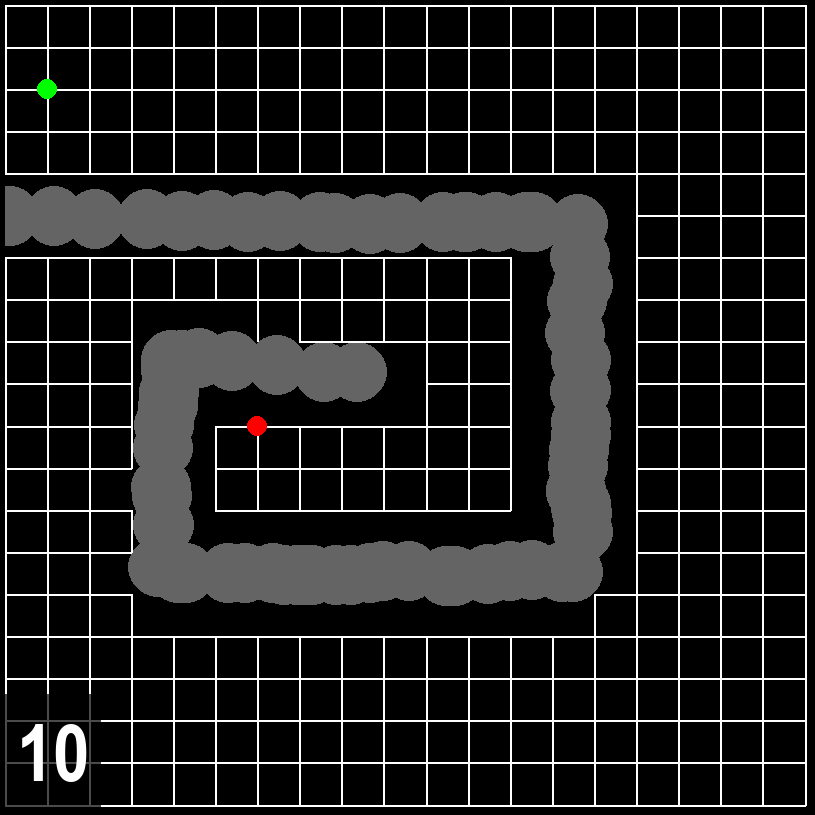}
        \includegraphics[width=0.18\linewidth]{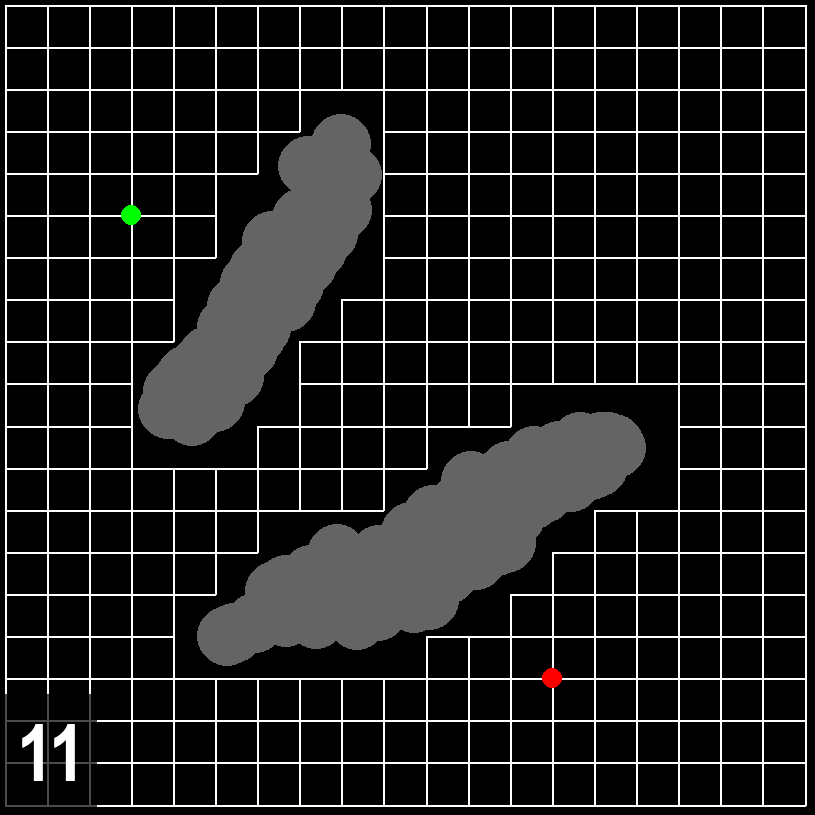}
        \includegraphics[width=0.18\linewidth]{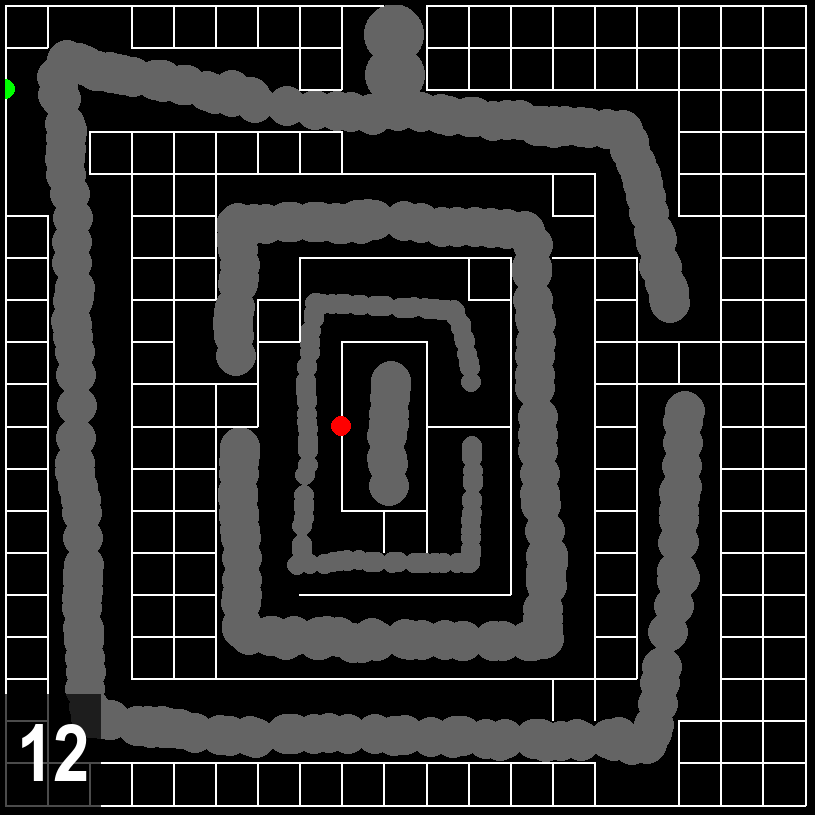}
        \includegraphics[width=0.18\linewidth]{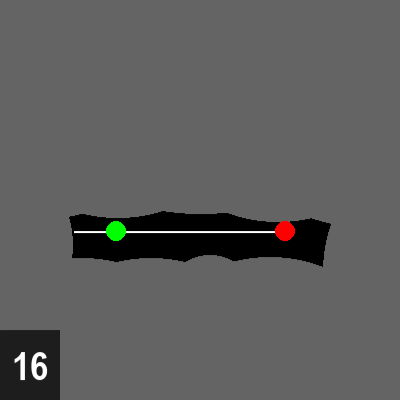}
        \vspace{-1.25em}
    \caption{Ten representative problems (enumerated from top-left to bottom-right as 0, 1, 3, 4, 8, 9, 10, 11, 12, 16). The green dot is the initial state. The red dot is the goal state. The white segments connect the states and can be viewed as the actions allowed by the robot. All actions have the same step cost. The gray areas are obstacles.}
    \label{fig:problems}
    \vspace{-1.5em}
\end{figure}

Run times are shown, which closely correlate with number of actions taken (these appear in the appendix).
We ran each algorithm $100$ times and report the mean and standard deviation.
Unless otherwise stated, the number of episodes used for \ql{} is set to $1000$, and the number of steps per episode is set to $3000$ to be sufficiently longer than any optimal path.
These were chosen through trial and error to optimize Q-learning performance.
Each run was terminated either after convergence or an iteration limit was reached.
We considered three questions: (1) Was a path from the initial state to a goal state found? (2) Was the path found optimal (is the cost-to-go value from the initial state optimal)? (3) Is the cost-to-go value for each state in the state space optimal?

The first question  
corresponds to a feasible
planning solution: An action sequence that leads to the goal.  Figure~\ref{fig:epsilon_discovery} shows how varying $\epsilon$ affects the time in which the goal state is first discovered.
When 
all state-action pairs (except the goal) are initialized to the same value,
a greedy action would be chosen arbitrarily (depending on which action is first in its enumeration) because all neighboring values are equal.
Later, as the cost of any action that has been chosen in previous iterations increases relative to the other available actions, unexplored actions would be favored.
This happens until the goal is discovered and better values are propagated backwards, causing 
actions making progress towards a lower cost solution to be preferred.
For problems in which there are less states and more obstacles (e.g., Problem 12), the goal is discovered more quickly with smaller values for $\epsilon$.
For Problem 0, since the whole state space is available and the goal is relatively far away, a purely greedy approach ($\epsilon = 0$) takes more time to reach it.
However, a random policy ($\epsilon=1$) is quicker because the robot is not hindered by any obstacles.

\begin{figure}[h!]
\vspace{-1.5em}
    \centering
        \includegraphics[width=0.49\linewidth]{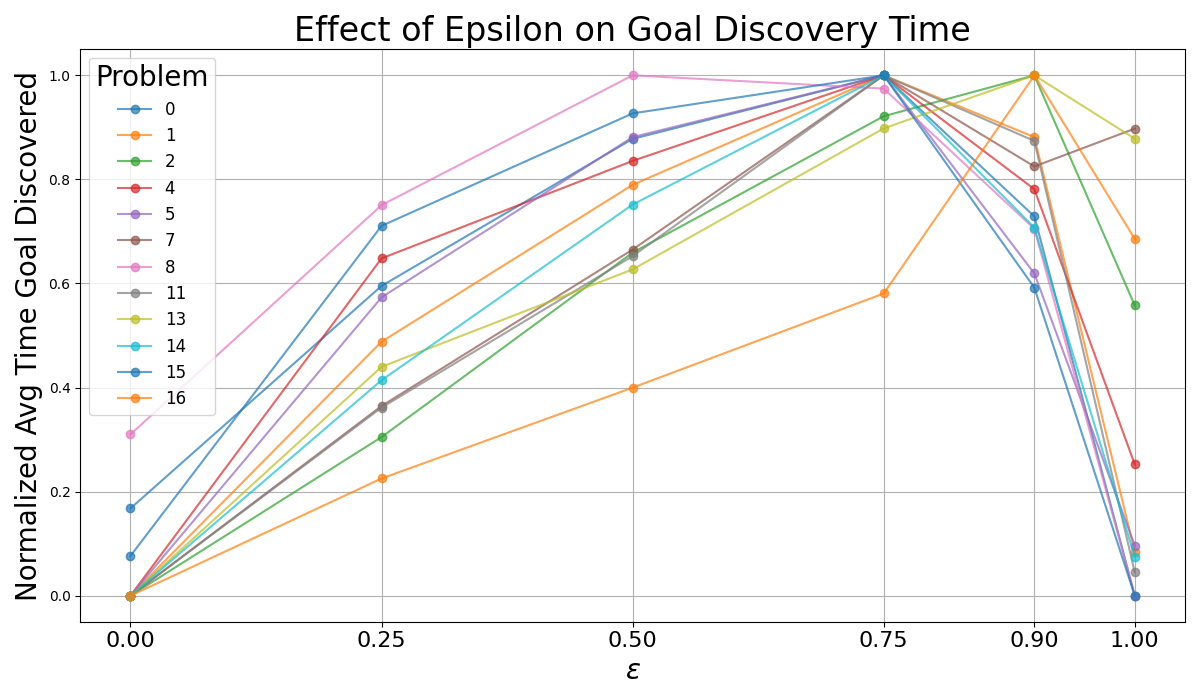}
        \includegraphics[width=0.49\linewidth]{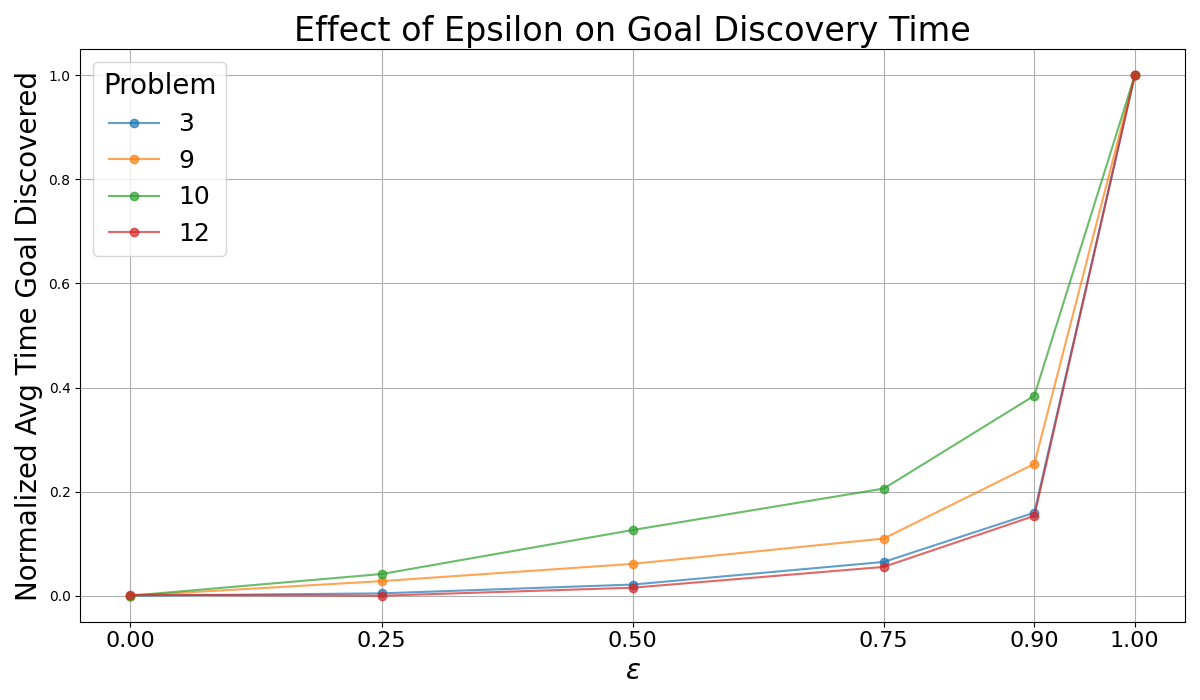}
\vspace{-1.2em}
    \caption{The effect of decreasing greediness in a policy when applied on different problems.}
    \label{fig:epsilon_discovery}
\vspace{-1.75em}  
\end{figure}

To address questions (2) and (3) above, we define convergence for \ql{} as convergence to the optimal cost-to-go values computed by a model-based value iteration.
Note that $Q^*$ can be obtained from $G^*$ using (\ref{eqn:qval2}) for this comparison.
Table~\ref{tab:prob10det} shows the performance of the model-free versions of Dijkstra's algorithm and value iteration, 
and that of \ql{} with varying $\epsilon$ for a representative problem.
We also include a deterministic version of \ql{}, labeled as `$(\uppi)$', that uses digits of $\uppi=3.14\dots$ base four, to select actions \cite{TimLavLav25}.
We show the average run time of the algorithm, how often the values for the whole state space converge (frequency of convergence), and the time it takes for the value of the initial state to converge.

\begin{table}[!h]
\caption{Performance of \ql{} and model-free approaches on a deterministic problem (Problem 10). Value iteration shortened to VI.}
\vspace{-0.5em}
\label{tab:prob10det}
\centering
\resizebox{\textwidth}{!}{
\begin{tabular}{l@{\hspace{7pt}}c@{\hspace{7pt}}c@{\hspace{7pt}}c}
\hline
 Algorithm (Problem 10) & \makecell{Run time\\(mean $\pm$ std)} & Convergence & \makecell{Optimal Initial\\Cost-to-Go Time\\(mean $\pm$ std)} \\
\hline
Q-learning ($\epsilon=0$) & 0.33383 $\pm$ 0.0025 & 0.0 \% &0.1720 $\pm$ 0.0016 \\
Q-learning ($\epsilon=0.25$) & 0.44088 $\pm$ 0.0030 & 0.0 \% &0.2043 $\pm$ 0.0041 \\
Q-learning ($\epsilon=0.5$) & 0.64394 $\pm$ 0.0045 & 0.0 \% &0.2517 $\pm$ 0.0066 \\
Q-learning ($\epsilon=0.75$) & 1.08726 $\pm$ 0.2099 & 76.0 \% &0.3502 $\pm$ 0.0129 \\
Q-learning ($\epsilon=0.9$) & 0.61889 $\pm$ 0.0559 & 100.0 \% &0.4858 $\pm$ 0.0239 \\
Q-learning ($\epsilon=1$) & 1.88356 $\pm$ 0.6740 & 100.0 \% &1.2836 $\pm$ 0.3801 \\
Q-learning\;($\pi$) ($\epsilon=1)$ & 2.02001 $\pm$ 0.4996 & 100.0 \% &1.3425 $\pm$ 0.2770 \\
Model-free Dijkstra & 0.00248 $\pm$ 0.0006 & N/A &N/A \\
Model-free Asynchronous VI & 0.04920 $\pm$ 0.0004 & N/A &N/A \\
Model-free VI & 0.05585 $\pm$ 0.0007 & N/A &N/A \\
\hline
\end{tabular}}
\vspace{-1.2em}
\end{table}

Unsurprisingly, the model-free Dijkstra algorithm has the fastest run time, followed by model-free asynchronous and then synchronous value iteration (these run times include graph exploration).
The fastest version of \ql{} is purely greedy ($\epsilon = 0$).
As a result, we can see that the more randomness we introduce in the form of increasing $\epsilon$, the slower the algorithm runs.
However, there is a reverse relationship on the convergence rate: lower $\epsilon$ decreases the frequency of convergence.  This aligns with intuition because less exploration will prevent states from being visited sufficiently often with high probability.
However, as $\epsilon$ increases, both the convergence frequency and run time increase.
Interestingly, 
for lower $\epsilon$ values,  
the optimal cost-to-go value from the initial state converges faster. Based on this observation, if the objective
is to simply find an optimal path between the initial and the goal states when the system is deterministic, picking $\epsilon = 0$ is often the best choice.

The model-free version of Dijkstra's algorithm is on average 134.65 times faster and uses 22.88 times fewer actions than \ql{} with $\epsilon = 0$.
If the objective is not only to find an optimal path from the initial state, but to find the optimal cost-to-go values for all states, we choose $\epsilon = 0.9$, since this is the value for which \ql{} converges consistently with the lowest run time.
In this case, model-free Dijkstra's algorithm is 249.62 times faster and uses 46.50 times fewer actions, which is an even bigger difference in performance. 
We observed that the relative performance among the different \ql{} variants and the model-free approaches is consistent across all problems studied, although the magnitude of the performance differences varies.

\section{Analysis of Cost/Reward Models}\label{sec:cost}
\vspace{-0.75em}
This section brings the planning-oriented model of Section \ref{sec:det} closer to the models typically used in RL, which involve rewards, discounting, and episodes.  We mathematically analyze and experimentally study the effect of each, while remaining in a deterministic setting to improve clarity; Section \ref{sec:stoch} will then consider the case of stochastic systems.

\vspace{-1em}
\subsection{Cost and reward models are equivalent, almost}\label{sec:cr}

Engineers have tended to minimize {\em cost} whereas AI researchers, especially in RL, have preferred to maximize {\em reward}.  The former usually involves physical quantities, such as time, energy, distance, or perhaps money in a business setting.  The latter takes inspiration from biology by imagining the robot or agent as being conditioned by Pavlovian rewards.  As shown below, these formulations are equivalent at a general level; however, the problem with biologically inspired reward is that it tends to lead to the heuristic invention of reward functions to `motivate' the robot to act in preconceived ways \cite{NgHarRus99}.  We therefore introduce the idea of {\em truecost}, in which the aim is to formulate costs that directly map to the physics of a given system (or perhaps money in a business setting).  Ideally, costs or rewards should not be tweaked until the algorithm does what is was supposed to do.  Of course, we can imagine {\em truereward}, but we emphasize that it should be more concrete, such as profit (revenue minus cost) in a business setting, rather than a vague bio-inspired motivator.

It is generally known that replacing the costs in a problem formulation by the corresponding rewards (multiplying the costs by ${-1}$) and switching from minimization to maximization results in an equivalent problem formulation in terms of the optimal policies. We show that this is generally true for any cost functional that is linear in terms of the sequence of stepwise costs. Note that this criterion covers all the standard cost formulations, including total cost, discounted cost, asymptotic average cost and the expected values of these in the stochastic formulations.

\begin{proposition}\label{prop:cost_reward_eq}
Let $P = (X, U, f, x_1, X_G, \ell)$ be a planning problem with a cost functional $L_\textrm{c} : U^\Nat \times X \to \Re$. Assume that the cost-functional can be written as $L_\textrm{c}(\tilde{u},x_1) = \mathcal{L}\big(h(\tilde{u},x_1)\big)$ in which $h : U^\Nat \times X \to \Re^\Nat$ maps a path onto the corresponding sequence of stepwise (immediate) costs and $\mathcal{L} : \Re^\Nat \to \Re$ is a linear operator. Then a policy $\pi$ attains a minimum for $L_\textrm{c}$ if and only if it attains a maximum for the operator $L_\textrm{r}$ defined by $L_\textrm{r}(\tilde{u},x_1) = \mathcal{L}\big(-h(\tilde{u},x_1)\big)$.
\end{proposition}
\begin{proof}
For $x \in X$, let $\tilde{u}^* \in U^\Nat$ be the action sequence corresponding to following an $L_\textrm{c}$-optimal policy $\pi^*$ from $x$. Then $L_\textrm{c}(\tilde{u}^*, x) \leq L_\textrm{c}(\tilde{u}, x)$ for all $\tilde{u} \neq \tilde{u}^*$ so that
\[
L_\textrm{r}(\tilde{u}^*, x) = \mathcal{L}\big(-h(\tilde{u}^*, x)\big) \geq \mathcal{L}\big(-h(\tilde{u}, x)\big) = L_\textrm{r}(\tilde{u}, x)
\]
for all $\tilde{u} \neq \tilde{u}^*$. This means that applying $\tilde{u}^*$ from $x$ corresponds to maximizing the reward functional $L_\textrm{r}$, in other words, following an $L_\textrm{r}$-optimal policy. The implication in the other direction follows similarly, and since $x$ was arbitrary, the claim follows. \qed
\end{proof}

\subsection{The dangers of discounting}\label{sec:discount}
Infinite horizon decision making problems tend to yield diverging sums of costs or rewards.  If there is a goal, as in the case of planning, then the simplest way to fix it is to use a termination action as introduced in Section \ref{sec:det}: No more costs/rewards accumulate after termination, thereby forcing (\ref{eqn:cf}) to be finite if the goal is reachable.  The preferred way in RL is to use discounted cost/reward, which is a kind of mathematical hack dating back to Bellman \cite{Bel57}.  For any $\alpha \in (0,1)$, the following sum is finite:
\begin{equation}\label{eqn:disccost}
L_\alpha(\pi,x_1) = \lim_{K \to \infty}
\left( \sum_{k=1}^K \alpha^{k-1} \ell(x_k,u_k) \right).
\end{equation}
This implies that future costs/rewards are considered less valuable than current ones and it might be fine in highly unpredictable settings such as financial markets, but for robotics this causes a large and dangerous deviation from truecost. In the following, we discuss this aspect of discounting.

Let $\tilde{u} = (u_1, \ldots, u_{N-1})$ be a sequence of actions. A sequence of states $(x_1, \ldots, x_N)$ in which $x_{k+1} = f(x_k, u_k)$ for $k = 1, \ldots, N-1$ is a \emph{cycle}, if $x_1 = x_N$ and $x_k \neq x_1$ for $k = 2, \ldots, N-1$. We denote by $C(\tilde{u}, x)$ a cycle beginning at $x$ corresponding to $\tilde{u}$. If $\tilde{u}$ is given by a state-feedback policy $\pi:X\rightarrow U$, we denote a cycle starting at $x$ by $C_\pi(x)$. 
For the following proposition, suppose $\ell$ is a strictly positive cost function other than $\ell(x,u_T)$ which is $0$ for any $x$. Without loss of generality, we assume that there is a single goal state $x_G$.

\begin{proposition} \label{prop:cost_cycle} Given a planning problem $\P=(X,U,f,x_I,x_G, \ell)$, 
suppose there exists a policy $\pi_{G}$ for which $L(\pi_G,x_I)$, the cost according to the cost functional in \eqref{eqn:cf}, is finite. Furthermore, suppose there exists a $\pi_C$ such that the path $(x_1=x_I, \pi_C(x_1),x_2,\pi_C(x_2),\dots)$ contains a cycle $C_\pi(x)$, in which $x=x_m$, for some $m$. Then, there exist $\ell$ and $\alpha$ such that $L( \pi^*_\alpha, x_1) = \infty$, in which $\pi^*_\alpha$ is an optimal policy minimizing the cost functional defined in \eqref{eqn:disccost}.
\end{proposition}
\begin{proof}
Without loss of generality, suppose that the cycle given by $\pi_C$ starts at $x_I$, that is, $C_{\pi_C}(x_I)=(x_1=x_I, x_2, \dots, x_{N+1}=x_I)$ for some $N$ and suppose that $C_{\pi_C}$ does not include $x_G$. Since we are considering an infinite-horizon setting, this cycle will be executed infinitely many times. Let $\bar{c}= c_1 + \alpha c_2 + \alpha^2 c_3 + \dots \alpha^{N-1}c_{N}$ be the cost that accumulates going through this cycle once, in which $c_i=\ell(x_i,\pi_C(x_i))$, $i=1,\dots,N$. 
Let $\bar{\alpha}:=\alpha^{N}$. The discounted cost associated with $\pi_C$ is 
$
L_\alpha(\pi_C,x_I)=\lim_{K\to\infty}\sum_{k=0}^K (\bar{\alpha})^{k}\bar{c} = \frac{\bar{c}}{1-\bar{\alpha}}.
$
Let $\pi^*_G$ be a policy that minimizes \eqref{eqn:disccost} among the policies that terminate at goal. The total discounted cost for $\pi^*_G$ is $L_\alpha(\pi^*_G,x_I)=\lim_{K\to\infty}\sum_{k=1}^{K} {\alpha}^{k-1}\ell(x_k,\pi^*_G(x_k))=\sum_{k=1}^{M}\alpha^{k-1}\ell(x_k,\pi^*_G(x_k))$, since the cost of termination action $u_T$, applied at the goal state $x_k=x_G$ for $k>M-1$ for some $M$, is 0. Denote by $\pi^*_\alpha$, an optimal policy minimizing \eqref{eqn:disccost}. 
Then, $L_\alpha(\pi^*_\alpha, x_I) \leq L_\alpha(\pi_C, x_I)$ and $L_\alpha(\pi^*_\alpha, x_I) \leq L_\alpha(\pi^*_G,x_I)$. If  $L_\alpha(\pi^*_\alpha, x_I)=\frac{\bar{c}}{1-\bar{\alpha}} < \sum_{k=1}^{K-1}\alpha^{k-1}\ell(x_k,\pi^*_G(x_k))=L_\alpha(\pi^*_G,x_I)$, $\pi^*_\alpha$ must contain a cycle. Then, we can pick $\ell$ and $\alpha$ such that this inequality is satisfied. In that case, $L(\pi^*_\alpha,x_I)=\infty$ since it contains a cycle. \qed

\end{proof}

To illustrate the issue that an optimal solution to a discounted cost problem might 
miss the goal even if the goal is reachable, consider the planning problem with $X=\{0, 1, \dots, 5\}$, $x_G=5$, $U=\{-1,1\}$, $f(x,u)=\min(5, \max(0, x+u))$, and $\ell(x,u)=x+1$. The only policy that reaches the goal is $\pi_G$ for which $\pi_G(x)=1$ for any $x \in X\setminus \{5\}$ and $\pi_G(5)=u_T$. For this policy, the truecost from $x_1=0$ is $L(\pi_G,0)=15$. For the discounted formulation, 
let 
$\pi_C(x)=-1$ for any $x \in X\setminus \{5\}$ and $\pi_C(5)=u_T$. 
Then, if $\alpha$ is selected so that $1/(1-\alpha) < \sum_{i=1}^5\alpha^{i-1}(i)$, an optimal policy $\pi^*_\alpha$ for the discounted cost formulation will have a cycle and miss the goal. 

A more optimistic interpretation of this result is that a careful choice of a discount factor could prevent from such danger. However, we highlight that discounting is a heuristic that is as problematic as using potential functions for motion planning and becoming trapped in local minima \cite{BarLanLat92}. One could try to pick $\alpha$ as close to one as possible, hoping to fend off the problem of ignoring the goal. 
An alternative way to ensure that the cost remains bounded over an infinite horizon is to optimize the average cost/reward per stage, which is also traced back to Bellman \cite{Bel57}, and has been more recently suggested for RL \cite{Sch93,Mah96}.  
The relationship between this and termination actions is the basis of Section \ref{sec:episode}.

\subsection{Episodic equivalences}\label{sec:episode}

Imagine that the robot must solve the same tasks over and over, forever.  Each time it arrives in the goal, it gets magically teleported back to the initial state and must get to the goal again; this is an example of physical jumping from Section~\ref{sec:det}.  This leads to an infinite-horizon problem that fits the model of Section \ref{sec:det} but is closer to the common RL formulation that relies upon repeated {\em episodes} \cite{SutBar98}.  In this section, we establish an equivalence between the two formulations.  This is worth considering because the resulting insight extends to problems in which different goal requests may be made in each episode (see, for example, \cite{ShaLavHut96}).

Consider the following two problem formulations: 
\begin{itemize}
    \item[(i)] Let
    $P_\textrm{unsp} = (X, U, f, x_I, X_G, \ell)$
    be an \emph{unspecified-horizon} problem. We assume there exists a termination action $u_T$ for which the terminal cost is $\ell(x_G, u_T) = 0$. 
    The cost functional $L$ is given by~\eqref{eqn:cf}.
    \item[(ii)] Let
        $P_\textrm{inf} = (X, U, f, x_I, X_G, \ell)$
    be an \emph{infinite-horizon} problem. The cost functional $L_\textrm{av} : U^\Nat \to \Re$ is defined as the asymptotic average
    \begin{equation} \label{Eq_Def_asymptotic_cost}
    L_\textrm{av}(\tilde{u}, x_1) := \lim_{N \to \infty}\frac{1}{N} \sum_{k=1}^N \ell(x_k, u_k).
    \end{equation}
    Instead of setting 
    $\ell(x_G, u_T) = 0$
    as in (i), we assume that the system resets to $x_I$ whenever it reaches $x_G$ and incurs a negative cost (a bonus) $M < 0$.
\end{itemize}

\begin{figure}[t!]
\centering
\begin{minipage}{0.32\linewidth}
\includegraphics[width=0.98\linewidth]{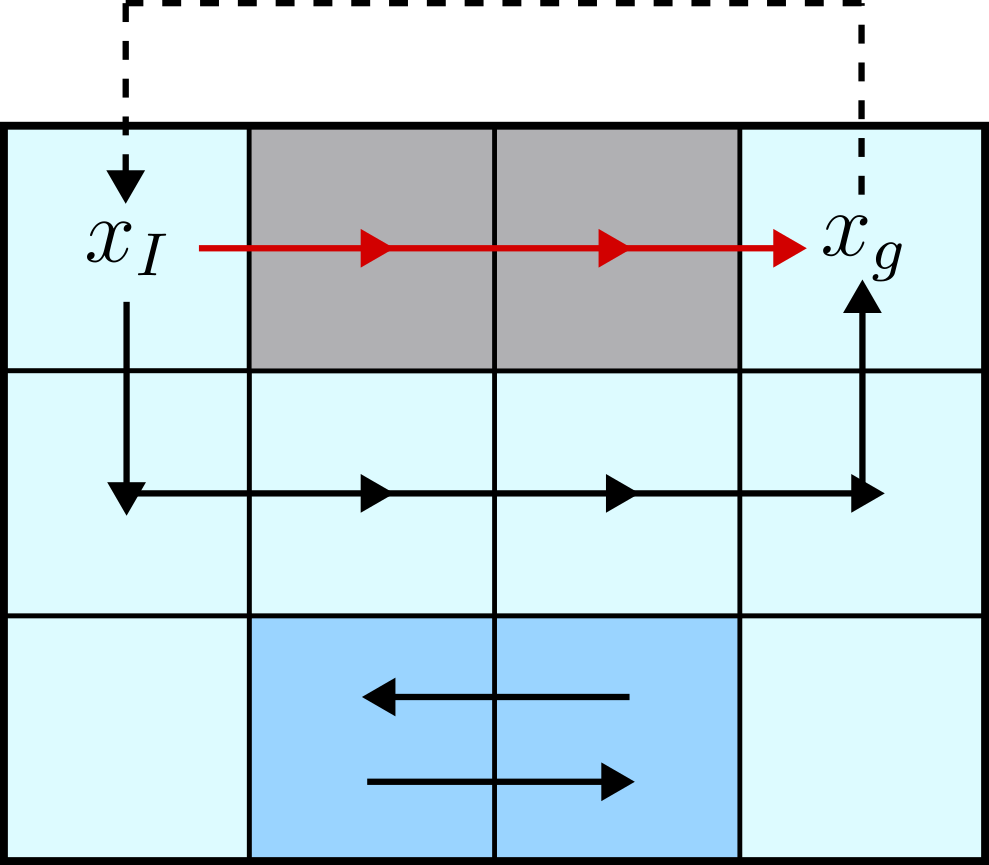}
\end{minipage}
\hfill
\begin{minipage}{0.67\linewidth}
\caption{The difference between 
a single trial (unspecified horizon) 
versus infinite-horizon formulation. The optimal path for a single trial is indicated with black arrows  
from $x_I$ to $x_G$. A more costly, but shorter, path is indicated with red arrows. In the infinite-horizon formulation, a shortcut (dashed arrow)  
takes the robot from $x_G$ and returns it to $x_I$. The dark blue squares indicate a region in state space where the robot could loop indefinitely with negligible average cost per step without ever reaching the goal.}   
\end{minipage}
\label{fig:equivalence}
\vspace{-1.2em}
\end{figure}

We characterize the planning problems for which the above definitions (i) and (ii) share an optimal policy. For $P_\textrm{unsp}$ there exists an optimal state-feedback policy $\pi_\textrm{u}^*$, which takes the robot from any initial state to the termination state with minimal cost.
In $P_\textrm{inf}$ there is no termination state, 
so any optimal policy leads the robot to enter a cycle with the lowest asymptotic average cost and to loop there indefinitely.
Due to the reset, every cycle containing $x_G$ must also contain $x_I$. It follows then from the optimality of $\pi_\textrm{u}^*$ that the lowest-cost cycle containing $x_G$ is $C_{\pi_\textrm{u}^*}(x_I)$. However, there may exist other cycles that do not contain $x_G$ and that have a lower asymptotic cost than $L_\textrm{av}(C_{\pi_\textrm{u}^*}(x_I))$. In such cases $\pi_\textrm{u}^*$ fails to be optimal for $P_\textrm{inf}$. This can be remedied by tuning the bonus parameter $M$ so that its cumulative effect overcomes the difference. 
However, changing the magnitude of $M$ may also have the effect that some cycle $C(\tilde{u}, x_I) \neq C_{\pi_\textrm{u}^*}(x_I)$ containing $x_G$ may become optimal under the asymptotic average-cost formulation due to the resetting. See Figure~\ref{fig:equivalence} for an illustration of the two types of cycles. The conclusion is that the formulations $P_\textrm{unsp}$ and $P_\textrm{inf}$ are not generally equivalent, but may be so if 
certain conditions regarding the lengths and average costs of cycles in $P_\textrm{inf}$ are satisfied. This is made precise in the following.

For an action sequence $\tilde{u}_k := (u_1, \ldots, u_{k-1})$, denote by $\Phi \big(x, \tilde{u}_k\big)$ the path $(x_1, \ldots, x_k)$ in which $x_1 = x$ and $x_{j+1} = f(x_j, u_j)$ for all $j = 1, \ldots, k-1$. Denote by $\varphi(x, \tilde{u}_k)$ the last state in $\Phi(x, \tilde{u}_k)$.
Let
$
U_G := \{ \tilde{u} \, : \, \varphi(x_I, \tilde{u}) = x_G  \}
$
be the set of action sequences that take the robot from $x_I$  to $x_G$. 
The policies for $P_\textrm{inf}$ can be divided into those that correspond to applying some $\tilde{u} \in U_G$ from $x_I$, and those that tend towards some other cycle that does not contain $x_G$. For any policy $\pi$ that corresponds to looping an action sequence $\tilde{u} \in U_G$ from the initial state $x_I$, the asymptotic average cost is given by 
\begin{equation} \label{Eq_Def_Cycle_cost}
L_\textrm{av}(\tilde{u}, x_I) 
= \frac{M + L(\tilde{u}, x_I)}{S(\tilde{u})}
= \frac{1}{S(\tilde{u})} \Bigg( M + \sum_{k=1}^{S(\tilde{u})-1} \ell \big(x_k, f(x_k, u_k) \big) \Bigg)
\end{equation}
in which $S(\tilde{u})$ denotes the length of the path $\Phi \big(x_I, \tilde{u}\big)$. 
Let $\tilde{u}^* \in U_G$ be the action sequence corresponding to applying the optimal policy for $P_\textrm{unsp}$ at $x_I$. Define for each $\tilde{u} \in U_G$ the value
\begin{equation} \label{Eq_Def_constraint_function}
\mathcal{A}(\tilde{u}) := \bigg(1 - \frac{S(\tilde{u})}{S(\tilde{u}^*)} \bigg) \bigg( \frac{S(\tilde{u}) L(\tilde{u}^*, x_I)}{S(\tilde{u}^*)L(\tilde{u}, x_I)} - 1 \bigg)
\end{equation}
and the sets
$
\mathcal{R}^- := \big\{ \tilde{u} \,:\, S(\tilde{u}) < S(\tilde{u}^*) \big\}$,
$
\mathcal{R}^+ := \big\{ \tilde{u} \,:\, S(\tilde{u}) > S(\tilde{u})^* \big\}.$
Furthermore, let $\hat{C} = C(\tilde{u}, x_1)$ be the cycle which has, among all the cycles that do not contain $x_G$, the smallest asymptotic average cost
\begin{align} \label{Eq_Def_shortest_cycle_cost}
L_\textrm{av}(\hat{C}) = L_\textrm{av}(\tilde{u}, x_1)
&= \frac{L(\tilde{u}, x_1)}{S(\tilde{u})-1}.
\end{align}

\begin{proposition}\label{prop:episodic_equiv}
Let $\pi^*$ be an optimal policy for $P_\textrm{unsp}$. Then $\pi^*$ is also an optimal policy for $P_\textrm{inf}$ with reset bonus $M<0$ if and only if $M$ satisfies
\[
a \leq M \leq \min \big\{b, S(\tilde{u}^*)L_\textrm{av}(\hat{C}) - L(\tilde{u}^*, x_I) \big\}, 
\]
in which $\tilde{u}^*$ is the action sequence corresponding to applying $\pi^*$ from $x_I$ to reach $x_G$, $\hat{C}$ is as in~\eqref{Eq_Def_shortest_cycle_cost}, $\mathcal{A}(\tilde{u})$ is as in~\eqref{Eq_Def_constraint_function}, and
\[
a := \max_{\tilde{u} \in \mathcal{R}^-} \{\mathcal{A}(\tilde{u})\}, \quad \textrm{and} \quad b := \min_{\tilde{u} \in \mathcal{R}^+} \{\mathcal{A}(\tilde{u})\}.
\]
\end{proposition}
\begin{proof}
For $\pi^*$ to be optimal for $P_\textrm{inf}$, the asymptotic average cost $L_\textrm{av}(\tilde{u}^*, x_I) = \big(M + L(\tilde{u}^*, x_I)\big) / S(\tilde{u}^*)$ must attain the lowest value among both the action sequences $\tilde{u} \in U_G$ and $\tilde{u} \notin U_G$. The second condition holds if and only if the cost $L_\textrm{av}(\tilde{u}^*, x_I)$ is less than $L_\textrm{av}(\hat{C})$, which can be expressed as
\begin{align*}
M < S(\tilde{u}) L_\textrm{av}(\hat{C}) - L(\tilde{u}, x_I).
\end{align*}
For $\tilde{u}^*$ to yield the lowest asymptotic average cost among the action sequences $\tilde{u} \in U_G$, it needs to satisfy the inequality
\begin{equation} \label{Eq_optimality_condition}
\frac{L(\tilde{u}^*, x_I) + M}{S(\tilde{u}^*)} \leq \frac{L(\tilde{u}, x_I) + M}{S(\tilde{u})}
\end{equation}
for every $\tilde{u} \in U_G$. This is equivalent to $M \leq \mathcal{A}(\tilde{u})$ for $\tilde{u} \in \mathcal{R}^+$ and to $M \geq \mathcal{A}(\tilde{u})$ for $\tilde{u} \in \mathcal{R}^-$. Thus, for $\tilde{u}^*$ to satisfy~\eqref{Eq_optimality_condition} relative to all $\tilde{u} \in U_G$ is equivalent to $M$ satisfying the double inequality $\max_{\tilde{u} \in \mathcal{R}^-} \{\mathcal{A}(\tilde{u})\} \leq M \leq \min_{\tilde{u} \in \mathcal{R}^+} \{\mathcal{A}(\tilde{u})\}$. \qed
\end{proof}

\section{From Deterministic to Stochastic Models}\label{sec:stoch}

\subsection{Basic assumptions and approach}

The concepts from Sections \ref{sec:det} and \ref{sec:cost} extend gracefully to the stochastic setting by replacing $f$ in $\P$ with a probabilistic transition model $P(x_{k+1}|x_k,u_k)$, which denotes the probability that $x_{k+1}$ is reached if $u_k$ is applied from $x_k$.  This causes paths taken during execution to be unpredictable, resulting in $G^*$ and $Q^*$ representing {\em expected} costs.  Furthermore, every state-action pair should be visited an infinite number of times to learn $P$, rather than once to learn $f$ from the deterministic setting. Thus, the method from Section \ref{sec:det} of quickly learning $f$ first, and then applying Dijkstra's algorithm is no longer applicable.  At best, one could try to estimate the transition probabilities to some level of statistical confidence and then apply stochastic value iteration to obtain $G^*$ and $\pi^*$.  For comparative purposes, we will consider the case of stochastic value iteration applied to a perfectly learned transition model.  However, most of our studies in this section are based on Q-learning.

In our implementations, the transition model $P(x_{k+1}|x_k,u_k)$ is defined to be a direct extension of $f$ from Section \ref{sec:det}, in terms of a {\em predictability factor}, $\gamma \in (0,1]$, that controls the amount of entropy or prediction uncertainty.  From each $x_k$, consider applying $u_k$ to get $x_{k+1} = f(x_k,u_k)$.  For the stochastic version with predictability factor $\gamma$, we define $P(x'_{k+1}|x_k,u_k) = \gamma$ if $x'_{k+1} = x_{k+1}$; otherwise, the remaining probability mass $1 - \gamma$ is spread uniformly as if nature causes the remaining actions (excluding $u_k$) to be chosen from $U(x_k)$ (including a termination action that holds the state constant).  If $\gamma = 1$, then the deterministic model can essentially be simulated.  If $\gamma = 1/2$, then on average, half of the time the robot moves in the direction it was commanded, and for the other half it effectively executes an alternative action randomly.

Another critical parameter is the learning rate $\rho$, which appears in (\ref{eqn:qvalit}).  Imagine implementing a complementary filter that uses $\rho$ to interpolate between a noisy data source and the current best estimate.  If the data is perfectly reliable, as in the deterministic case of Section \ref{sec:det}, then we can set $\rho = 1$.  As the predictability decreases (lower $\gamma$), we would expect to use lower values for $\rho$ so that the noise is attenuated and the estimate becomes more stable.  Choosing a good value of $\rho$ is a delicate art.  We considered multiple values in our experiments, and furthermore developed an adaptive tuning method for $\rho$ so that it decreases during execution according to the formula
\begin{equation} \label{eq:decaying_lr}
 \rho(x,u) = \frac{1}{n(x,u)^{\omega}},
\end{equation}
in which $n(x,u)$ is the number of times that a state-action pair has been visited and $\omega$ is an arbitrary parameter (this idea was introduced in \cite{EveMan01} and used in the Double \ql{} approach \cite{Has10}).
In this approach, there is a separate $\rho(x,u)$ for each state-action pair, which is decreased as visits accumulate.

\subsection{Computational comparisons}

The experimental setup outlined in Section~\ref{sec:det_comp_comparisons} is reused with a few minor changes.
In cases where $\gamma \leq 0.9$ we use $10$ samples instead of $100$. 
We increased the number of episodes to 3000 to improve the convergence frequency, that is, the percentage of times the values converge.  We investigated performance by trading off three parameters: learning rate $\rho$, greedy parameter $\epsilon$, and predictability factor $\gamma$.  We studied convergence rates in comparison to the true optimal cost to go, obtained from stochastic value iteration, and flagged cases in which the Q-values are within $10\%$ of optimal to identify approximate successes.

For cases of high $\gamma$ (more predictable), we investigated combinations of $\epsilon, \rho \in \{0.5, 0.7, 0.9, 0.99\}$.
The $\rho$ values were chosen after careful investigation and observing that for $\rho < 0.5$ and $\rho > 0.99$, the results were not more informative.  We also found that increasing $\rho$ decreases run time.
In Table~\ref{tab:stoch_gamma_0999}, which combines the data for Problems 1 and 10, note that for $\gamma = 0.999$ there is less benefit of being greedy in comparison to the results of Section \ref{sec:det_comp_comparisons}.
Although choosing $\epsilon = 0$ still leads to a better or equal run time to $\epsilon = 1$, we see that $\epsilon=0.9$ has a much better performance than both, while also converging. This appears to be due the fact that nondeterminism throws the state away from greedy routes, thereby requiring more exploration to get values to converge.
Furthermore, we can see that DP methods converge about two orders of magnitude more quickly than RL, which reflects the price of learning on the fly.
Similar trends were observed in other problems; the appendix contains many more.  It also reports statistics for 
the goal discovery time, which are similar to the deterministic case, but grow with increasing $\rho$.

\begin{table}[!t]
\caption{Performance of \ql{} and DP methods for Problems 1 and 10 for stochastic problem with $\gamma = 0.999$. Note the addition of two columns: how often the value for the initial state converges; the shortest and longest path achieved over all samples.}
\label{tab:stoch_gamma_0999}
\resizebox{\textwidth}{!}{
\begin{tabular}{lccccc}
\hline
\makecell{Algorithm \\ ($\gamma= 0.999$)} & \makecell{Run time \\ (mean $\pm$ std)}& Convergence & \makecell{Optimal Initial \\ Cost2Go Time \\ (mean $\pm$ std)} & \makecell{Optimal Initial \\ Cost2Go \\ Convergence}& \makecell{Shortest/Longest \\ Path} \\
\hline
(Pr 1) Q-learning ($\rho=0.99, \epsilon=0$) & 0.40677 $\pm$ 0.0041 & 0.0\% & 0.0888 $\pm$ 0.0008 & 100.0\% & 28/32 \\
(Pr 1) Q-learning ($\rho=0.99, \epsilon=0.25$) & 0.52970 $\pm$ 0.0064 & 0.0\% & 0.1006 $\pm$ 0.0024 & 100.0\% & 28/29 \\
(Pr 1) Q-learning ($\rho=0.99, \epsilon=0.5$) & 0.76335 $\pm$ 0.0084 & 0.0\% & 0.1144 $\pm$ 0.0024 & 100.0\% & 28/30 \\
(Pr 1) Q-learning ($\rho=0.99, \epsilon=0.75$) & 0.52489 $\pm$ 0.1930 & 100.0\% & 0.1516 $\pm$ 0.0045 & 100.0\% & 28/30 \\
(Pr 1) Q-learning ($\rho=0.99, \epsilon=0.9$) & 0.27856 $\pm$ 0.0271 & 100.0\% & 0.2044 $\pm$ 0.0109 & 100.0\% & 28/29 \\
(Pr 1) Q-learning ($\rho=0.99, \epsilon=1$) & 0.40572 $\pm$ 0.0491 & 100.0\% & 0.3512 $\pm$ 0.0450 & 100.0\% & 28/30 \\
(Pr 1) Async Value Iteration & 0.12869 $\pm$ 0.0009 & N/A & N/A & N/A & 28/29 \\
(Pr 1) Value Iteration & 0.17870 $\pm$ 0.0024 & N/A & N/A & N/A & 28/30 \\
(Pr 10) Q-learning ($\rho=0.99, \epsilon=0$) & 0.94648 $\pm$ 0.0174 & 0.0\% & 0.2290 $\pm$ 0.0314 & 100.0\% & 64/27731 \\
(Pr 10) Q-learning ($\rho=0.99, \epsilon=0.25$) & 1.24982 $\pm$ 0.0203 & 0.0\% & 0.2513 $\pm$ 0.0056 & 100.0\% & 64/66 \\
(Pr 10) Q-learning ($\rho=0.99, \epsilon=0.5$) & 1.91614 $\pm$ 0.0188 & 0.0\% & 0.3029 $\pm$ 0.0058 & 100.0\% & 64/67 \\
(Pr 10) Q-learning ($\rho=0.99, \epsilon=0.75$) & 0.97765 $\pm$ 0.2121 & 100.0\% & 0.4122 $\pm$ 0.0107 & 100.0\% & 64/66 \\
(Pr 10) Q-learning ($\rho=0.99, \epsilon=0.9$) & 0.65742 $\pm$ 0.0572 & 100.0\% & 0.5611 $\pm$ 0.0225 & 100.0\% & 64/66 \\
(Pr 10) Q-learning ($\rho=0.99, \epsilon=1$) & 2.26297 $\pm$ 0.7142 & 100.0\% & 1.6362 $\pm$ 0.4014 & 100.0\% & 64/66 \\
(Pr 10) Async Value Iteration & 0.21926 $\pm$ 0.0162 & N/A & N/A & N/A & 64/66 \\
(Pr 10) Value Iteration & 0.35463 $\pm$ 0.0252 & N/A & N/A & N/A & 64/66 \\
\hline
\end{tabular}}
\vspace{-1em}
\end{table}

Lowering $\gamma$ to $0.99$ prevents convergence, even if we increase the episode count to $5000$ or $10000$.
In fact, further investigation reveals that the average distance between the optimal values and the values obtained with \ql{} remains nearly constant as the number of episodes increase when $\epsilon=1$. This suggests that convergence happens on a state-space level, but to values different than the optimal ones.
However, the initial cost-to-go values still converge for some parameter combinations.
Table~\ref{tab:prob1and10stoch099} shows when this happens for Problems 1 and 10.
The general observation is that a lower $\rho$ yields more reliable convergence because it is more robust to unpredictability, but requires more iterations. As $\gamma$ is lowered, then $\rho$ should be lowered accordingly, although the best combination is found only by trial and error.
It is worthy noting that a policy that leads the robot to the goal in a nearly optimal manner can be computed even if the Q-values from the initial state do not converge, as seen in the configuration with $\rho = \epsilon = 0.5$ for Problem 1. Interestingly, decreasing $\gamma$, we observe, for Problems 1 and 11, that the run time of value iteration becomes closer to that of greedy \ql{} with a small $\rho$.

\begin{table}[!h]
\caption{Performance of \ql{} and DP methods for Problems 1 and 10 for stochastic problem with $\gamma = 0.99$.}
\label{tab:prob1and10stoch099}
\resizebox{\textwidth}{!}{
\begin{tabular}{lccccc}
\hline
\makecell{Algorithm \\ ($\gamma= 0.99$)} & \makecell{Run time \\ (mean $\pm$ std)}& Convergence & \makecell{Optimal Initial \\ Cost2Go Time \\ (mean $\pm$ std)} & \makecell{Optimal Initial \\ Cost2Go \\ Convergence}& \makecell{Shortest/Longest \\ Path} \\
\hline
(Pr 1) Q-learning ($\rho=0.5, \epsilon=0$) & 0.44775 $\pm$ 0.0094 & 0.0\% & 0.1815 $\pm$ 0.0082 & 100.0\% & 28/34 \\
(Pr 1) Q-learning ($\rho=0.99, \epsilon=0$) & 0.42034 $\pm$ 0.0090 & 0.0\% & 0.2078 $\pm$ 0.0750 & 86.0\% & 28/2205 \\
(Pr 1) Q-learning ($\rho=0.5, \epsilon=0.5$) & 0.82059 $\pm$ 0.0169 & 0.0\% & 0.5143 $\pm$ 0.1522 & 32.0\% & 28/31 \\
(Pr 1) Q-learning ($\rho=0.99, \epsilon=0.5$) & 0.77585 $\pm$ 0.0109 & 0.0\% & nan $\pm$ nan & 0.0\% & 28/1474 \\
(Pr 1) Q-learning ($\rho=0.5, \epsilon=1$) & 15.46274 $\pm$ 0.1914 & 0.0\% & 5.2842 $\pm$ 3.5571 & 90.0\% & 28/33 \\
(Pr 1) Q-learning ($\rho=0.99, \epsilon=1$) & 15.39603 $\pm$ 0.1846 & 0.0\% & nan $\pm$ nan & 0.0\% & 28/32 \\
(Pr 1) Async Value Iteration & 0.15656 $\pm$ 0.0014 & N/A & N/A & N/A & 28/33 \\
(Pr 1) Value Iteration & 0.22934 $\pm$ 0.0020 & N/A & N/A & N/A & 28/31 \\
(Pr 10) Q-learning ($\rho=0.5, \epsilon=0$) & 1.09601 $\pm$ 0.0226 & 0.0\% & 0.4524 $\pm$ 0.0133 & 100.0\% & 64/322 \\
(Pr 10) Q-learning ($\rho=0.99, \epsilon=0$) & 1.07131 $\pm$ 0.0178 & 0.0\% & 0.2651 $\pm$ 0.0000 & 1.0\% & 66/2596 \\
(Pr 10) Q-learning ($\rho=0.5, \epsilon=0.5$) & 2.08346 $\pm$ 0.0157 & 0.0\% & 0.7959 $\pm$ 0.2102 & 100.0\% & 64/68 \\
(Pr 10) Q-learning ($\rho=0.99, \epsilon=0.5$) & 1.99043 $\pm$ 0.0217 & 0.0\% & 0.6288 $\pm$ 0.3471 & 99.0\% & 64/3689 \\
(Pr 10) Q-learning ($\rho=0.5, \epsilon=1$) & 30.93903 $\pm$ 0.2059 & 0.0\% & 3.5574 $\pm$ 1.1094 & 100.0\% & 64/70 \\
(Pr 10) Q-learning ($\rho=0.99, \epsilon=1$) & 30.94423 $\pm$ 0.1896 & 0.0\% & 2.4456 $\pm$ 1.0780 & 100.0\% & 64/610 \\
(Pr 10) Async Value Iteration & 0.24640 $\pm$ 0.0036 & N/A & N/A & N/A & 64/69 \\
(Pr 10) Value Iteration & 0.40702 $\pm$ 0.0028 & N/A & N/A & N/A & 64/70 \\
\hline
\end{tabular}}
\end{table}

To study global state-space convergence to optimal values, we investigated the effects of a small $\rho$ on simpler problems (i.e., Problem 8 and Problem 16).
For $\gamma = 0.9$, a random policy ($\epsilon = 1$) with $\rho = 0.1$ converges for both problems, whereas for Problem 16 even a greedy policy ($\epsilon = 0$) with the same $\rho$ value suffices.
However, this approach is not sufficient for $\gamma = 0.7$.
Instead, global convergence with $\epsilon = 1$ occurs only by decaying $\rho$ as described in (\ref{eq:decaying_lr}), using $\omega=0.7$.
As state-action visits increase, $\rho$ decays at a rate controlled by $\omega$, stabilizing the estimate.
We observe global convergence for Problem 16 with $\gamma = 0.5$ using the same approach.
We were unable to get global convergence in Problem 8 by any choice of $\omega$.

\section{Conclusions}\label{sec:con}

Our work has brought planning and reinforcement learning into closer alignment so that approaches, assumptions, and models across both can be compared and be better understood.  Through the introduction of a fully deterministic version of RL, we established its convergence and studied its performance relative to dynamic-programming based planning.  Through analysis of cost models, we have warned against using discounted cost, as is popular in RL; instead, we have advocated for truecost, in which the costs/rewards translate directly into physical or monetary values, rather than using them as a way to tune performance.  We have also established equivalences of maximizing rewards and minimizing costs, and episodic models to single-shot goal satisfaction models.  These provide motivation to use undiscounted cost models and termination actions for RL for solving goal-oriented tasks.  We applied these models in the stochastic system setting and showed how they extend naturally from the deterministic case, but inherit unique RL-oriented challenges, such as greediness and learning rate parameter tuning, which warrant further study.  We are currently extending the mathematical analysis of cost-reward and episodic equivalences to the stochastic case.

\subsubsection*{Acknowledgments} 
We thank Markku Suomalainen for helpful feedback and discussions.

\clearpage

\bibliographystyle{plain}
\bibliography{references}

\clearpage
\appendix
\section{Problems}
\begin{figure}[h!]
    \centering

        \includegraphics[width=0.18\linewidth]{Example_16_small.png}
        \includegraphics[width=0.18\linewidth]{Example_8_small.png}
        \includegraphics[width=0.18\linewidth]{Example_0.png}
        \includegraphics[width=0.18\linewidth]{Example_1.png}
        \includegraphics[width=0.18\linewidth]{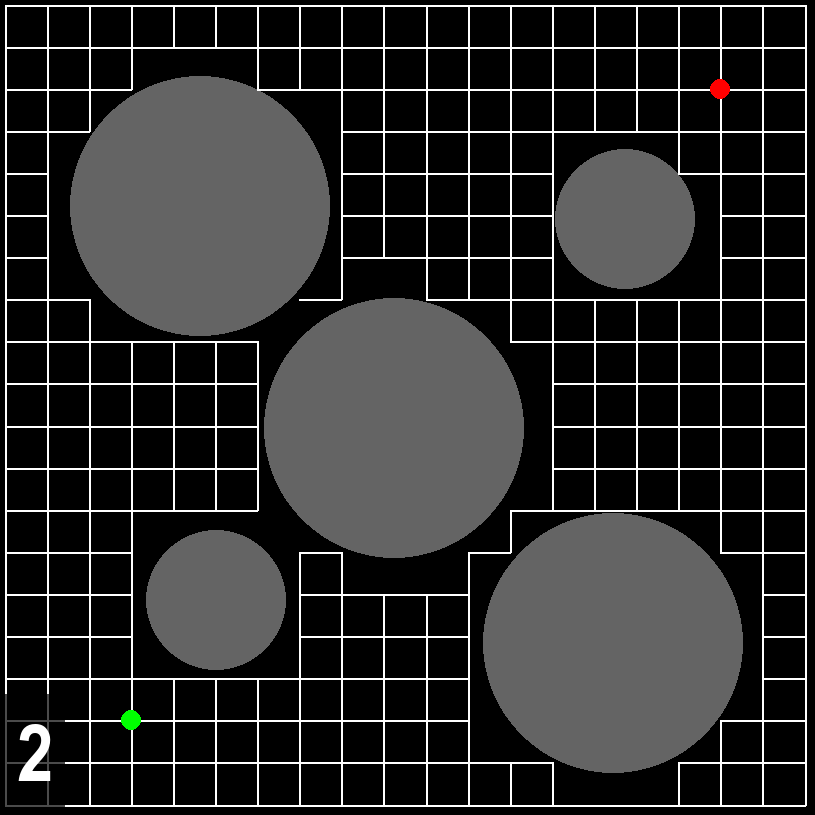}

    \vspace{0.3cm}

        \includegraphics[width=0.18\linewidth]{Example_3.png}
        \includegraphics[width=0.18\linewidth]{Example_4.png}
        \includegraphics[width=0.18\linewidth]{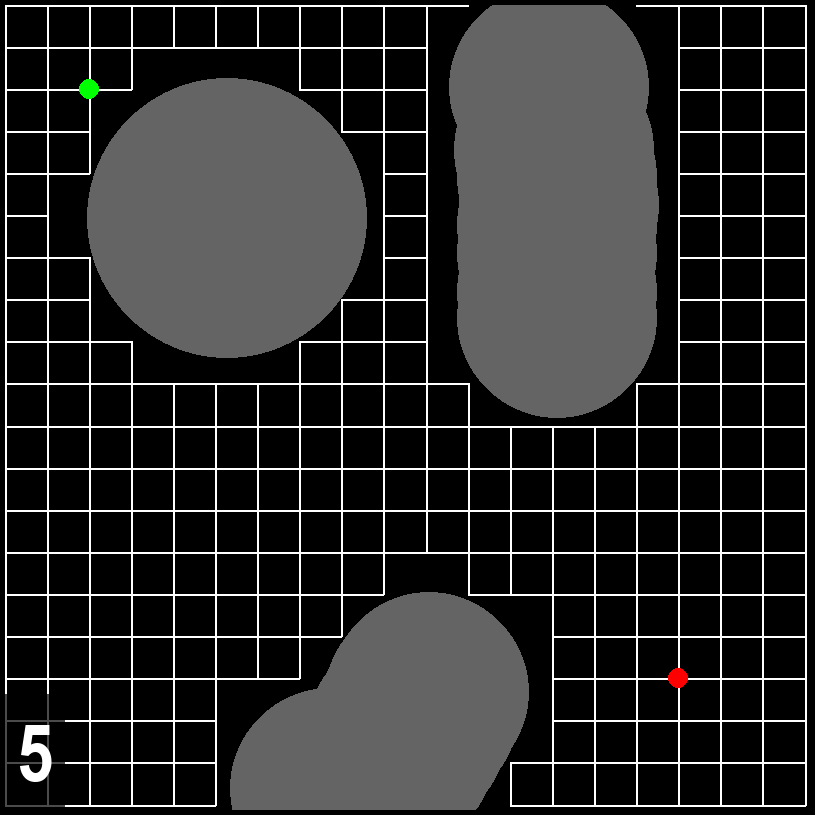}
        \includegraphics[width=0.18\linewidth]{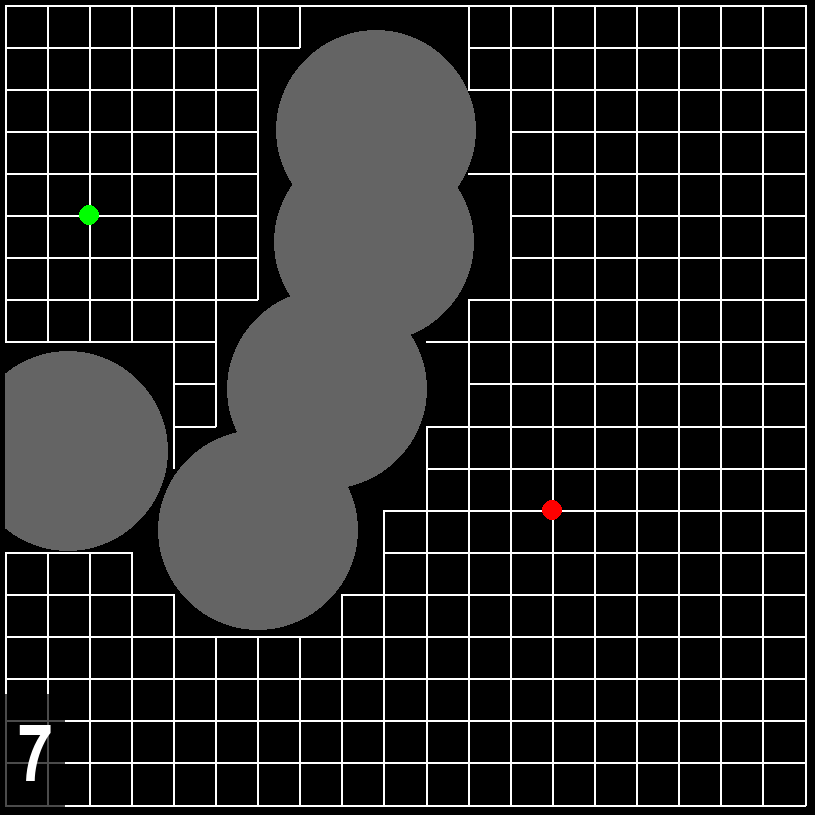}
        \includegraphics[width=0.18\linewidth]{Example_9.png}

    \caption{Ten problems (enumerated from top-left to bottom-right as 16, 8, 0, 1, 2, 3, 4, 5, 7, 9). The green dot is the initial state, while the red dot is the goal state. The white lines connect the states and can be viewed as the actions allowed by the robot. The gray circles represent obstacles that cannot be passed through}
\end{figure}

\begin{figure}[h!]
    \centering

        \includegraphics[width=0.25\linewidth]{Example_10.png}
        \includegraphics[width=0.25\linewidth]{Example_11.png}
        \includegraphics[width=0.25\linewidth]{Example_12.png}

    \vspace{0.3cm}

        \includegraphics[width=0.25\linewidth]{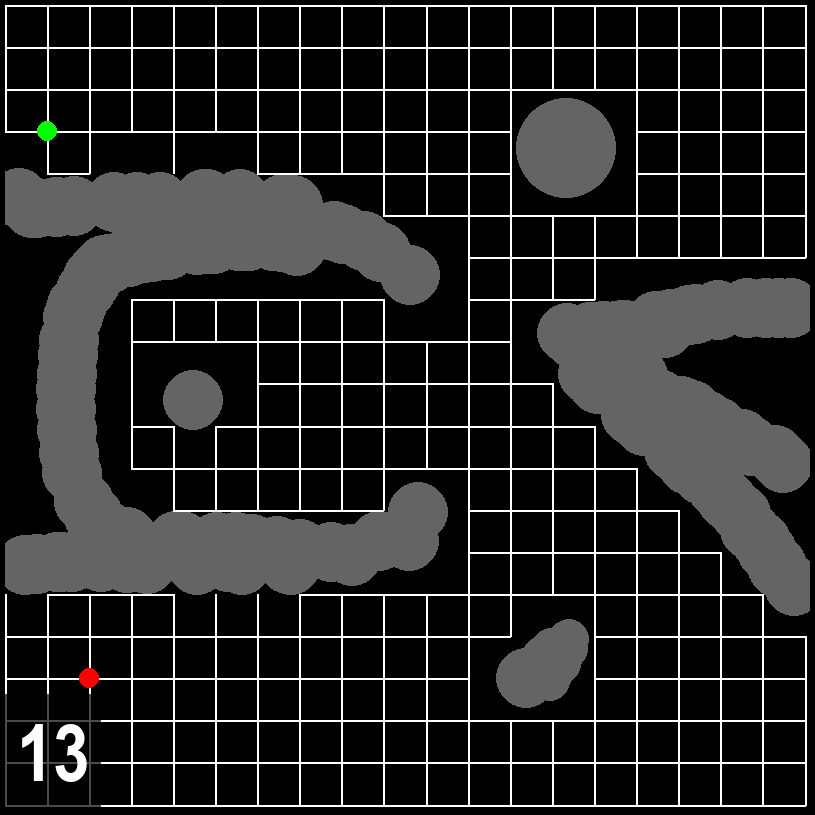}
        \includegraphics[width=0.25\linewidth]{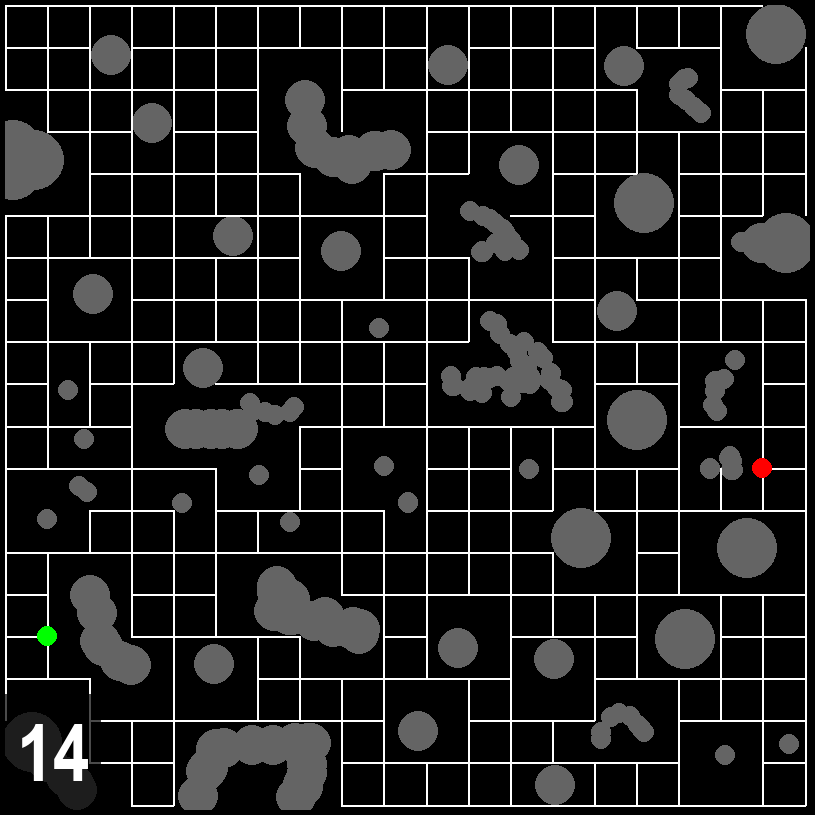}
        \includegraphics[width=0.25\linewidth]{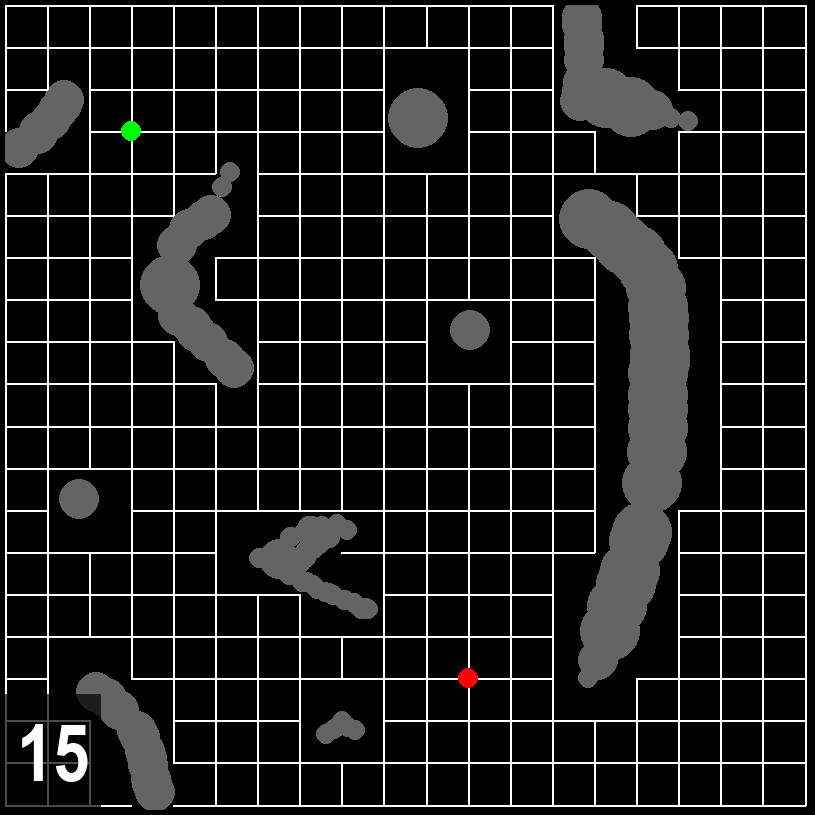}

    \caption{Additional six problems  (enumerated from top-left to bottom-right as 10 - 15).}
\end{figure}
\section{Tables for Deterministic Experiments}
The performance for \ql{} and the dynamic programming methods on deterministic problems, run on a desktop computer with a Ryzen 7 3700X CPU (3.6GHz), is shown in the following tables. We show the run times of the algorithms and the number of actions that were used during that run time. We also show how often the algorithm converges and how long it took for its exploration policy to discover the goal for the first time, together with the time it took for the initial state to achieve its optimal cost to go value, tracking the mean and standard deviation for each applicable metric.
The last three tables of this subsection show the same results for Problems 1, 10 and 11 run on a desktop computer with an i5-12400F CPU (2.50GHz), but also including a row for the performance of \ql{} when using $\epsilon$-decay.

\begin{table}[!h]
\resizebox{0.75\paperwidth}{!}{
\hspace*{-7.7cm}
\begin{tabular}{lccccccc}
\hline
Algorithm (Problem 0) & Run time (mean $\pm$ std) & \#actions (mean $\pm$ std) & Convergence & Discover goal time (mean $\pm$ std) & Discover goal \#actions (mean $\pm$ std) & Optimal Initial Cost2Go Time (mean $\pm$ std) & Optimal Initial Cost2Go \#actions (mean $\pm$ std) \\
\hline
Q-learning ($\epsilon=0$) & 0.11326 $\pm$ 0.0007 & 32774.0000 $\pm$ 0.0000 & 0.0 \% &0.0075 $\pm$ 0.0000 & 2131.0000 $\pm$ 0.0000 & 0.0622 $\pm$ 0.0003 & 18000.0000 $\pm$ 0.0000 \\
Q-learning ($\epsilon=0.25$) & 0.15533 $\pm$ 0.0019 & 45186.3700 $\pm$ 270.0393 & 0.0 \% &0.0138 $\pm$ 0.0019 & 4041.3800 $\pm$ 537.1843 & 0.0691 $\pm$ 0.0014 & 20140.0000 $\pm$ 374.6999 \\
Q-learning ($\epsilon=0.5$) & 0.24147 $\pm$ 0.0025 & 72453.8400 $\pm$ 618.0135 & 0.0 \% &0.0181 $\pm$ 0.0029 & 5436.8200 $\pm$ 871.1261 & 0.0831 $\pm$ 0.0023 & 24990.0000 $\pm$ 655.6676 \\
Q-learning ($\epsilon=0.75$) & 0.58168 $\pm$ 0.0079 & 180569.8300 $\pm$ 2367.9754 & 0.0 \% &0.0199 $\pm$ 0.0048 & 6167.2200 $\pm$ 1481.6033 & 0.1117 $\pm$ 0.0038 & 34790.0000 $\pm$ 1210.7436 \\
Q-learning ($\epsilon=0.9$) & 0.95443 $\pm$ 0.1644 & 303767.8100 $\pm$ 52175.4497 & 100.0 \% &0.0158 $\pm$ 0.0081 & 5039.6800 $\pm$ 2570.6394 & 0.1504 $\pm$ 0.0076 & 48030.0000 $\pm$ 2451.3466 \\
Q-learning ($\epsilon=1$) & 0.40398 $\pm$ 0.0332 & 135913.8400 $\pm$ 11113.1169 & 100.0 \% &0.0050 $\pm$ 0.0039 & 1636.8800 $\pm$ 1303.7126 & 0.2082 $\pm$ 0.0192 & 69980.0000 $\pm$ 6452.8753 \\
Q-learning\;($\pi$) ($\epsilon=1)$ & 0.41524 $\pm$ 0.0313 & 136706.6400 $\pm$ 10296.1120 & 100.0 \% &0.0042 $\pm$ 0.0034 & 1359.0800 $\pm$ 1101.2584 & 0.2114 $\pm$ 0.0187 & 69530.0000 $\pm$ 6119.5670 \\
\hline
\end{tabular}}
\end{table}
\vspace*{-1cm}
\begin{table}[!h]
\resizebox{0.75\paperwidth}{!}{
\hspace*{-7.7cm}\begin{tabular}{lccccccc}
\hline
Algorithm (Problem 1) & Run time (mean $\pm$ std) & \#actions (mean $\pm$ std) & Convergence & Discover goal time (mean $\pm$ std) & Discover goal \#actions (mean $\pm$ std) & Optimal Initial Cost2Go Time (mean $\pm$ std) & Optimal Initial Cost2Go \#actions (mean $\pm$ std) \\
\hline
Q-learning ($\epsilon=0$) & 0.13929 $\pm$ 0.0009 & 40894.0000 $\pm$ 0.0000 & 0.0 \% &0.0039 $\pm$ 0.0000 & 1127.0000 $\pm$ 0.0000 & 0.0688 $\pm$ 0.0005 & 20000.0000 $\pm$ 0.0000 \\
Q-learning ($\epsilon=0.25$) & 0.18457 $\pm$ 0.0016 & 53343.1400 $\pm$ 228.6908 & 0.0 \% &0.0074 $\pm$ 0.0025 & 2141.1200 $\pm$ 731.3317 & 0.0817 $\pm$ 0.0018 & 23570.0000 $\pm$ 495.0758 \\
Q-learning ($\epsilon=0.5$) & 0.26569 $\pm$ 0.0025 & 79250.9600 $\pm$ 530.5786 & 0.0 \% &0.0096 $\pm$ 0.0031 & 2861.3800 $\pm$ 926.8256 & 0.0984 $\pm$ 0.0025 & 29390.0000 $\pm$ 719.6527 \\
Q-learning ($\epsilon=0.75$) & 0.54417 $\pm$ 0.0282 & 172628.4100 $\pm$ 8373.3798 & 9.0 \% &0.0111 $\pm$ 0.0037 & 3518.6200 $\pm$ 1163.6424 & 0.1321 $\pm$ 0.0058 & 42070.0000 $\pm$ 1779.0728 \\
Q-learning ($\epsilon=0.9$) & 0.28439 $\pm$ 0.0376 & 92948.4700 $\pm$ 12287.7324 & 100.0 \% &0.0102 $\pm$ 0.0064 & 3342.1800 $\pm$ 2101.8655 & 0.1823 $\pm$ 0.0126 & 59540.0000 $\pm$ 4016.0179 \\
Q-learning ($\epsilon=1$) & 0.33558 $\pm$ 0.0396 & 114579.3300 $\pm$ 13500.2663 & 100.0 \% &0.0045 $\pm$ 0.0040 & 1540.3600 $\pm$ 1391.8213 & 0.2837 $\pm$ 0.0396 & 96820.0000 $\pm$ 13717.4196 \\
Q-learning\;($\pi$) ($\epsilon=1)$ & 0.34452 $\pm$ 0.0405 & 114637.3700 $\pm$ 13326.4214 & 100.0 \% &0.0050 $\pm$ 0.0042 & 1643.4800 $\pm$ 1398.7443 & 0.2925 $\pm$ 0.0388 & 97310.0000 $\pm$ 12780.2152 \\
\hline
\end{tabular}}
\end{table}
\vspace*{-1cm}
\begin{table}[!h]
\resizebox{0.75\paperwidth}{!}{
\hspace*{-7.7cm}\begin{tabular}{lccccccc}
\hline
Algorithm (Problem 2) & Run time (mean $\pm$ std) & \#actions (mean $\pm$ std) & Convergence & Discover goal time (mean $\pm$ std) & Discover goal \#actions (mean $\pm$ std) & Optimal Initial Cost2Go Time (mean $\pm$ std) & Optimal Initial Cost2Go \#actions (mean $\pm$ std) \\
\hline
Q-learning ($\epsilon=0$) & 0.14180 $\pm$ 0.0012 & 43272.0000 $\pm$ 0.0000 & 0.0 \% &0.0030 $\pm$ 0.0000 & 905.0000 $\pm$ 0.0000 & 0.0662 $\pm$ 0.0004 & 20000.0000 $\pm$ 0.0000 \\
Q-learning ($\epsilon=0.25$) & 0.18772 $\pm$ 0.0014 & 57044.9800 $\pm$ 230.5382 & 0.0 \% &0.0051 $\pm$ 0.0022 & 1551.0000 $\pm$ 655.2924 & 0.0793 $\pm$ 0.0014 & 24020.0000 $\pm$ 399.4997 \\
Q-learning ($\epsilon=0.5$) & 0.27709 $\pm$ 0.0047 & 86112.3600 $\pm$ 627.8763 & 0.0 \% &0.0076 $\pm$ 0.0027 & 2356.4200 $\pm$ 837.0154 & 0.0963 $\pm$ 0.0027 & 29990.0000 $\pm$ 768.0495 \\
Q-learning ($\epsilon=0.75$) & 0.60526 $\pm$ 0.0143 & 194341.4700 $\pm$ 4266.8579 & 4.0 \% &0.0094 $\pm$ 0.0040 & 3016.5000 $\pm$ 1293.1793 & 0.1322 $\pm$ 0.0056 & 42610.0000 $\pm$ 1719.8546 \\
Q-learning ($\epsilon=0.9$) & 0.31286 $\pm$ 0.0407 & 103589.5700 $\pm$ 13507.4629 & 100.0 \% &0.0099 $\pm$ 0.0064 & 3264.3800 $\pm$ 2105.1116 & 0.1840 $\pm$ 0.0130 & 60960.0000 $\pm$ 4300.9766 \\
Q-learning ($\epsilon=1$) & 0.44531 $\pm$ 0.0903 & 153230.1400 $\pm$ 30413.1369 & 100.0 \% &0.0069 $\pm$ 0.0061 & 2365.7600 $\pm$ 2124.2876 & 0.3170 $\pm$ 0.0630 & 109180.0000 $\pm$ 21751.9562 \\
Q-learning\;($\pi$) ($\epsilon=1)$ & 0.45306 $\pm$ 0.0842 & 150127.5800 $\pm$ 27994.8262 & 100.0 \% &0.0079 $\pm$ 0.0059 & 2601.4600 $\pm$ 1943.0290 & 0.3359 $\pm$ 0.0605 & 111290.0000 $\pm$ 19900.3995 \\
\hline
\end{tabular}}
\end{table}
\vspace*{-1cm}
\begin{table}[!h]
\resizebox{0.75\paperwidth}{!}{
\hspace*{-7.7cm}\begin{tabular}{lccccccc}
\hline
Algorithm (Problem 3) & Run time (mean $\pm$ std) & \#actions (mean $\pm$ std) & Convergence & Discover goal time (mean $\pm$ std) & Discover goal \#actions (mean $\pm$ std) & Optimal Initial Cost2Go Time (mean $\pm$ std) & Optimal Initial Cost2Go \#actions (mean $\pm$ std) \\
\hline
Q-learning ($\epsilon=0$) & 0.31918 $\pm$ 0.0040 & 104562.0000 $\pm$ 0.0000 & 0.0 \% &0.0049 $\pm$ 0.0001 & 1602.0000 $\pm$ 0.0000 & 0.0792 $\pm$ 0.0012 & 26000.0000 $\pm$ 0.0000 \\
Q-learning ($\epsilon=0.25$) & 0.13560 $\pm$ 0.0222 & 43657.4800 $\pm$ 7142.5875 & 100.0 \% &0.0054 $\pm$ 0.0033 & 1733.7200 $\pm$ 1060.4745 & 0.0958 $\pm$ 0.0030 & 30940.0000 $\pm$ 822.4354 \\
Q-learning ($\epsilon=0.5$) & 0.13557 $\pm$ 0.0088 & 44921.7400 $\pm$ 2909.6461 & 100.0 \% &0.0074 $\pm$ 0.0046 & 2452.8400 $\pm$ 1545.4950 & 0.1175 $\pm$ 0.0039 & 38980.0000 $\pm$ 1224.5816 \\
Q-learning ($\epsilon=0.75$) & 0.17153 $\pm$ 0.0095 & 58706.1600 $\pm$ 3219.7965 & 100.0 \% &0.0126 $\pm$ 0.0069 & 4289.3400 $\pm$ 2352.6951 & 0.1590 $\pm$ 0.0077 & 54460.0000 $\pm$ 2566.7879 \\
Q-learning ($\epsilon=0.9$) & 0.22911 $\pm$ 0.0148 & 79609.0600 $\pm$ 5036.7241 & 100.0 \% &0.0238 $\pm$ 0.0140 & 8292.6800 $\pm$ 4890.7428 & 0.2145 $\pm$ 0.0137 & 74610.0000 $\pm$ 4780.9936 \\
Q-learning ($\epsilon=1$) & 0.99921 $\pm$ 0.3629 & 365387.0000 $\pm$ 132113.6869 & 100.0 \% &0.1230 $\pm$ 0.1161 & 45427.7600 $\pm$ 42713.2396 & 0.9842 $\pm$ 0.3651 & 359940.0000 $\pm$ 132921.3166 \\
Q-learning\;($\pi$) ($\epsilon=1)$ & 0.97154 $\pm$ 0.2853 & 334139.2400 $\pm$ 97320.2744 & 100.0 \% &0.1417 $\pm$ 0.1139 & 49120.8400 $\pm$ 38980.7772 & 0.9630 $\pm$ 0.2830 & 331240.0000 $\pm$ 96539.8488 \\
\hline
\end{tabular}}
\end{table}
\vspace*{-1cm}
\begin{table}[!h]
\resizebox{0.75\paperwidth}{!}{
\hspace*{-7.7cm}\begin{tabular}{lccccccc}
\hline
Algorithm (Problem 4) & Run time (mean $\pm$ std) & \#actions (mean $\pm$ std) & Convergence & Discover goal time (mean $\pm$ std) & Discover goal \#actions (mean $\pm$ std) & Optimal Initial Cost2Go Time (mean $\pm$ std) & Optimal Initial Cost2Go \#actions (mean $\pm$ std) \\
\hline
Q-learning ($\epsilon=0$) & 0.15125 $\pm$ 0.0034 & 46110.0000 $\pm$ 0.0000 & 0.0 \% &0.0044 $\pm$ 0.0003 & 1325.0000 $\pm$ 0.0000 & 0.0724 $\pm$ 0.0021 & 22000.0000 $\pm$ 0.0000 \\
Q-learning ($\epsilon=0.25$) & 0.19840 $\pm$ 0.0022 & 60374.1400 $\pm$ 315.4138 & 0.0 \% &0.0147 $\pm$ 0.0037 & 4472.6600 $\pm$ 1134.9649 & 0.0869 $\pm$ 0.0021 & 26480.0000 $\pm$ 538.1450 \\
Q-learning ($\epsilon=0.5$) & 0.28498 $\pm$ 0.0053 & 88913.2900 $\pm$ 649.2647 & 0.0 \% &0.0177 $\pm$ 0.0043 & 5566.3400 $\pm$ 1348.1139 & 0.1081 $\pm$ 0.0039 & 33890.0000 $\pm$ 1038.2196 \\
Q-learning ($\epsilon=0.75$) & 0.57889 $\pm$ 0.0082 & 186567.6800 $\pm$ 2407.6892 & 0.0 \% &0.0203 $\pm$ 0.0061 & 6607.7800 $\pm$ 1982.5550 & 0.1520 $\pm$ 0.0090 & 49330.0000 $\pm$ 2956.5351 \\
Q-learning ($\epsilon=0.9$) & 1.36798 $\pm$ 0.2875 & 451724.6400 $\pm$ 94716.2085 & 58.0 \% &0.0168 $\pm$ 0.0098 & 5610.9200 $\pm$ 3275.1360 & 0.2096 $\pm$ 0.0166 & 69800.0000 $\pm$ 5499.0908 \\
Q-learning ($\epsilon=1$) & 0.80289 $\pm$ 0.1904 & 278558.3600 $\pm$ 65909.6114 & 100.0 \% &0.0085 $\pm$ 0.0086 & 2930.4200 $\pm$ 2986.4306 & 0.3581 $\pm$ 0.0701 & 124770.0000 $\pm$ 24373.2866 \\
Q-learning\;($\pi$) ($\epsilon=1)$ & 0.81879 $\pm$ 0.2648 & 277054.3300 $\pm$ 88885.4437 & 100.0 \% &0.0076 $\pm$ 0.0073 & 2578.1200 $\pm$ 2461.4350 & 0.3647 $\pm$ 0.0628 & 124020.0000 $\pm$ 21413.5378 \\
\hline
\end{tabular}}
\end{table}
\vspace*{-1cm}
\begin{table}[!h]
\resizebox{0.75\paperwidth}{!}{
\hspace*{-7.7cm}\begin{tabular}{lccccccc}
\hline
Algorithm (Problem 5) & Run time (mean $\pm$ std) & \#actions (mean $\pm$ std) & Convergence & Discover goal time (mean $\pm$ std) & Discover goal \#actions (mean $\pm$ std) & Optimal Initial Cost2Go Time (mean $\pm$ std) & Optimal Initial Cost2Go \#actions (mean $\pm$ std) \\
\hline
Q-learning ($\epsilon=0$) & 0.14089 $\pm$ 0.0011 & 42386.0000 $\pm$ 0.0000 & 0.0 \% &0.0047 $\pm$ 0.0001 & 1430.0000 $\pm$ 0.0000 & 0.0625 $\pm$ 0.0004 & 19000.0000 $\pm$ 0.0000 \\
Q-learning ($\epsilon=0.25$) & 0.18808 $\pm$ 0.0012 & 56130.6800 $\pm$ 264.3807 & 0.0 \% &0.0107 $\pm$ 0.0029 & 3236.7400 $\pm$ 879.1809 & 0.0746 $\pm$ 0.0019 & 22440.0000 $\pm$ 571.3143 \\
Q-learning ($\epsilon=0.5$) & 0.27472 $\pm$ 0.0027 & 84451.7000 $\pm$ 702.8104 & 0.0 \% &0.0139 $\pm$ 0.0029 & 4341.1200 $\pm$ 894.5730 & 0.0940 $\pm$ 0.0044 & 29150.0000 $\pm$ 1336.9742 \\
Q-learning ($\epsilon=0.75$) & 0.58191 $\pm$ 0.0116 & 184877.3400 $\pm$ 2306.1895 & 0.0 \% &0.0152 $\pm$ 0.0044 & 4888.6000 $\pm$ 1401.3910 & 0.1350 $\pm$ 0.0120 & 43400.0000 $\pm$ 3778.8887 \\
Q-learning ($\epsilon=0.9$) & 0.98171 $\pm$ 0.2661 & 319656.4200 $\pm$ 86716.4391 & 98.0 \% &0.0112 $\pm$ 0.0065 & 3673.4600 $\pm$ 2140.2486 & 0.1834 $\pm$ 0.0206 & 60150.0000 $\pm$ 6788.7775 \\
Q-learning ($\epsilon=1$) & 0.49500 $\pm$ 0.1111 & 170925.8200 $\pm$ 38202.6919 & 100.0 \% &0.0058 $\pm$ 0.0048 & 1977.1400 $\pm$ 1657.1752 & 0.2829 $\pm$ 0.0489 & 97870.0000 $\pm$ 16768.8133 \\
Q-learning\;($\pi$) ($\epsilon=1)$ & 0.53800 $\pm$ 0.1353 & 179624.1400 $\pm$ 44847.8988 & 100.0 \% &0.0060 $\pm$ 0.0046 & 2018.1600 $\pm$ 1544.2775 & 0.2987 $\pm$ 0.0459 & 100000.0000 $\pm$ 15316.0047 \\
\hline
\end{tabular}}
\end{table}
\vspace*{-1cm}
\begin{table}[!h]
\resizebox{0.75\paperwidth}{!}{
\hspace*{-7.7cm}\begin{tabular}{lccccccc}
\hline
Algorithm (Problem 7) & Run time (mean $\pm$ std) & \#actions (mean $\pm$ std) & Convergence & Discover goal time (mean $\pm$ std) & Discover goal \#actions (mean $\pm$ std) & Optimal Initial Cost2Go Time (mean $\pm$ std) & Optimal Initial Cost2Go \#actions (mean $\pm$ std) \\
\hline
Q-learning ($\epsilon=0$) & 0.12079 $\pm$ 0.0007 & 37478.0000 $\pm$ 0.0000 & 0.0 \% &0.0029 $\pm$ 0.0000 & 880.0000 $\pm$ 0.0000 & 0.0459 $\pm$ 0.0004 & 14000.0000 $\pm$ 0.0000 \\
Q-learning ($\epsilon=0.25$) & 0.16239 $\pm$ 0.0016 & 49680.7600 $\pm$ 323.7531 & 0.0 \% &0.0050 $\pm$ 0.0018 & 1508.2000 $\pm$ 544.4474 & 0.0538 $\pm$ 0.0015 & 16270.0000 $\pm$ 443.9595 \\
Q-learning ($\epsilon=0.5$) & 0.24275 $\pm$ 0.0027 & 76035.4800 $\pm$ 661.8601 & 0.0 \% &0.0068 $\pm$ 0.0023 & 2071.0600 $\pm$ 695.9548 & 0.0669 $\pm$ 0.0022 & 20780.0000 $\pm$ 672.0119 \\
Q-learning ($\epsilon=0.75$) & 0.52271 $\pm$ 0.0068 & 169471.7200 $\pm$ 2123.6347 & 0.0 \% &0.0087 $\pm$ 0.0036 & 2803.9400 $\pm$ 1157.2607 & 0.0927 $\pm$ 0.0042 & 29970.0000 $\pm$ 1374.4453 \\
Q-learning ($\epsilon=0.9$) & 1.50883 $\pm$ 0.1231 & 499216.7200 $\pm$ 40690.9169 & 34.0 \% &0.0077 $\pm$ 0.0047 & 2541.0600 $\pm$ 1558.5223 & 0.1312 $\pm$ 0.0111 & 43450.0000 $\pm$ 3691.5444 \\
Q-learning ($\epsilon=1$) & 1.82401 $\pm$ 0.3753 & 633100.6200 $\pm$ 130431.7955 & 100.0 \% &0.0081 $\pm$ 0.0078 & 2820.5200 $\pm$ 2756.1423 & 0.3153 $\pm$ 0.0780 & 109720.0000 $\pm$ 27010.0278 \\
Q-learning\;($\pi$) ($\epsilon=1)$ & 1.96900 $\pm$ 0.6300 & 662812.2800 $\pm$ 211950.3269 & 100.0 \% &0.0095 $\pm$ 0.0107 & 3198.5800 $\pm$ 3640.9636 & 0.3005 $\pm$ 0.0670 & 101600.0000 $\pm$ 22493.5546 \\
\hline
\end{tabular}}
\end{table}
\vspace*{-1cm}
\begin{table}[!h]
\resizebox{0.75\paperwidth}{!}{
\hspace*{-7.7cm}\begin{tabular}{lccccccc}
\hline
Algorithm (Problem 8) & Run time (mean $\pm$ std) & \#actions (mean $\pm$ std) & Convergence & Discover goal time (mean $\pm$ std) & Discover goal \#actions (mean $\pm$ std) & Optimal Initial Cost2Go Time (mean $\pm$ std) & Optimal Initial Cost2Go \#actions (mean $\pm$ std) \\
\hline
Q-learning ($\epsilon=0$) & 0.01600 $\pm$ 0.0002 & 5094.0000 $\pm$ 0.0000 & 0.0 \% &0.0001 $\pm$ 0.0000 & 43.0000 $\pm$ 0.0000 & 0.0032 $\pm$ 0.0001 & 1000.0000 $\pm$ 0.0000 \\
Q-learning ($\epsilon=0.25$) & 0.02015 $\pm$ 0.0051 & 6466.8900 $\pm$ 1658.8178 & 26.0 \% &0.0002 $\pm$ 0.0000 & 51.9400 $\pm$ 9.1933 & 0.0032 $\pm$ 0.0000 & 1000.0000 $\pm$ 0.0000 \\
Q-learning ($\epsilon=0.5$) & 0.01306 $\pm$ 0.0078 & 4306.3500 $\pm$ 2580.3645 & 98.0 \% &0.0002 $\pm$ 0.0001 & 59.1200 $\pm$ 17.9840 & 0.0031 $\pm$ 0.0000 & 1000.0000 $\pm$ 0.0000 \\
Q-learning ($\epsilon=0.75$) & 0.00628 $\pm$ 0.0033 & 2127.4100 $\pm$ 1121.6373 & 100.0 \% &0.0002 $\pm$ 0.0001 & 60.1200 $\pm$ 28.4904 & 0.0030 $\pm$ 0.0000 & 1000.0000 $\pm$ 0.0000 \\
Q-learning ($\epsilon=0.9$) & 0.00510 $\pm$ 0.0025 & 1760.2800 $\pm$ 867.9923 & 100.0 \% &0.0002 $\pm$ 0.0001 & 54.9200 $\pm$ 36.0288 & 0.0030 $\pm$ 0.0000 & 1000.0000 $\pm$ 0.0000 \\
Q-learning ($\epsilon=1$) & 0.00458 $\pm$ 0.0020 & 1655.8500 $\pm$ 736.9191 & 100.0 \% &0.0001 $\pm$ 0.0001 & 40.5600 $\pm$ 31.1019 & 0.0028 $\pm$ 0.0000 & 1000.0000 $\pm$ 0.0000 \\
Q-learning\;($\pi$) ($\epsilon=1)$ & 0.00591 $\pm$ 0.0024 & 2022.0400 $\pm$ 856.0125 & 100.0 \% &0.0001 $\pm$ 0.0001 & 43.6600 $\pm$ 32.3469 & 0.0030 $\pm$ 0.0001 & 1000.0000 $\pm$ 0.0000 \\
\hline
\end{tabular}}
\end{table}
\vspace*{-1cm}
\begin{table}[!h]
\resizebox{0.75\paperwidth}{!}{
\hspace*{-7.7cm}\begin{tabular}{lccccccc}
\hline
Algorithm (Problem 9) & Run time (mean $\pm$ std) & \#actions (mean $\pm$ std) & Convergence & Discover goal time (mean $\pm$ std) & Discover goal \#actions (mean $\pm$ std) & Optimal Initial Cost2Go Time (mean $\pm$ std) & Optimal Initial Cost2Go \#actions (mean $\pm$ std) \\
\hline
Q-learning ($\epsilon=0$) & 0.30094 $\pm$ 0.0025 & 93512.0000 $\pm$ 0.0000 & 0.0 \% &0.0030 $\pm$ 0.0000 & 924.0000 $\pm$ 0.0000 & 0.1071 $\pm$ 0.0010 & 33000.0000 $\pm$ 0.0000 \\
Q-learning ($\epsilon=0.25$) & 0.40332 $\pm$ 0.0026 & 124056.6000 $\pm$ 369.5486 & 0.0 \% &0.0063 $\pm$ 0.0049 & 1920.4400 $\pm$ 1473.8903 & 0.1266 $\pm$ 0.0022 & 38730.0000 $\pm$ 580.6031 \\
Q-learning ($\epsilon=0.5$) & 0.59771 $\pm$ 0.0045 & 189069.5400 $\pm$ 805.8806 & 0.0 \% &0.0103 $\pm$ 0.0063 & 3219.8200 $\pm$ 1985.7043 & 0.1528 $\pm$ 0.0037 & 48240.0000 $\pm$ 1132.4310 \\
Q-learning ($\epsilon=0.75$) & 0.32806 $\pm$ 0.0497 & 106893.5400 $\pm$ 16033.8762 & 100.0 \% &0.0160 $\pm$ 0.0085 & 5189.9200 $\pm$ 2773.5873 & 0.2082 $\pm$ 0.0080 & 67830.0000 $\pm$ 2561.4644 \\
Q-learning ($\epsilon=0.9$) & 0.30673 $\pm$ 0.0195 & 102077.7200 $\pm$ 6343.1919 & 100.0 \% &0.0330 $\pm$ 0.0177 & 10934.4600 $\pm$ 5863.3670 & 0.2969 $\pm$ 0.0170 & 98810.0000 $\pm$ 5556.4287 \\
Q-learning ($\epsilon=1$) & 1.74662 $\pm$ 0.7671 & 611856.5000 $\pm$ 269104.1334 & 100.0 \% &0.1213 $\pm$ 0.1074 & 42628.7200 $\pm$ 37559.9343 & 1.2952 $\pm$ 0.4267 & 453590.0000 $\pm$ 149158.4456 \\
Q-learning\;($\pi$) ($\epsilon=1)$ & 2.10442 $\pm$ 1.1306 & 704462.5600 $\pm$ 377424.9032 & 100.0 \% &0.1779 $\pm$ 0.1123 & 59863.2600 $\pm$ 37681.3580 & 1.3004 $\pm$ 0.2590 & 435550.0000 $\pm$ 87139.9306 \\
\hline
\end{tabular}}
\end{table}
\vspace*{-1cm}
\begin{table}[!h]
\resizebox{0.75\paperwidth}{!}{
\hspace*{-7.7cm}\begin{tabular}{lccccccc}
\hline
Algorithm (Problem 10) & Run time (mean $\pm$ std) & \#actions (mean $\pm$ std) & Convergence & Discover goal time (mean $\pm$ std) & Discover goal \#actions (mean $\pm$ std) & Optimal Initial Cost2Go Time (mean $\pm$ std) & Optimal Initial Cost2Go \#actions (mean $\pm$ std) \\
\hline
Q-learning ($\epsilon=0$) & 0.33383 $\pm$ 0.0025 & 101214.0000 $\pm$ 0.0000 & 0.0 \% &0.0068 $\pm$ 0.0002 & 2015.0000 $\pm$ 0.0000 & 0.1720 $\pm$ 0.0016 & 52000.0000 $\pm$ 0.0000 \\
Q-learning ($\epsilon=0.25$) & 0.44088 $\pm$ 0.0030 & 133171.4600 $\pm$ 405.0955 & 0.0 \% &0.0101 $\pm$ 0.0058 & 3050.0200 $\pm$ 1747.4311 & 0.2043 $\pm$ 0.0041 & 61830.0000 $\pm$ 1200.4582 \\
Q-learning ($\epsilon=0.5$) & 0.64394 $\pm$ 0.0045 & 200028.6600 $\pm$ 996.3678 & 0.0 \% &0.0168 $\pm$ 0.0071 & 5222.4200 $\pm$ 2184.3765 & 0.2517 $\pm$ 0.0066 & 78540.0000 $\pm$ 1951.5122 \\
Q-learning ($\epsilon=0.75$) & 1.08726 $\pm$ 0.2099 & 348322.6400 $\pm$ 66374.2689 & 76.0 \% &0.0231 $\pm$ 0.0100 & 7449.8000 $\pm$ 3206.2196 & 0.3502 $\pm$ 0.0129 & 112990.0000 $\pm$ 3410.2639 \\
Q-learning ($\epsilon=0.9$) & 0.61889 $\pm$ 0.0559 & 205682.0500 $\pm$ 18488.2891 & 100.0 \% &0.0372 $\pm$ 0.0194 & 12384.0200 $\pm$ 6398.5418 & 0.4858 $\pm$ 0.0239 & 161590.0000 $\pm$ 7636.8776 \\
Q-learning ($\epsilon=1$) & 1.88356 $\pm$ 0.6740 & 657763.2100 $\pm$ 235315.6482 & 100.0 \% &0.0860 $\pm$ 0.0912 & 30118.4200 $\pm$ 32113.6083 & 1.2836 $\pm$ 0.3801 & 448790.0000 $\pm$ 133803.4600 \\
Q-learning\;($\pi$) ($\epsilon=1)$ & 2.02001 $\pm$ 0.4996 & 685022.7200 $\pm$ 169670.6191 & 100.0 \% &0.0539 $\pm$ 0.0569 & 18366.3400 $\pm$ 19499.6667 & 1.3425 $\pm$ 0.2770 & 455760.0000 $\pm$ 94888.6843 \\
\hline
\end{tabular}}
\end{table}
\vspace*{-1cm}
\begin{table}[!h]
\resizebox{0.75\paperwidth}{!}{
\hspace*{-7.7cm}\begin{tabular}{lccccccc}
\hline
Algorithm (Problem 11) & Run time (mean $\pm$ std) & \#actions (mean $\pm$ std) & Convergence & Discover goal time (mean $\pm$ std) & Discover goal \#actions (mean $\pm$ std) & Optimal Initial Cost2Go Time (mean $\pm$ std) & Optimal Initial Cost2Go \#actions (mean $\pm$ std) \\
\hline
Q-learning ($\epsilon=0$) & 0.11746 $\pm$ 0.0010 & 35316.0000 $\pm$ 0.0000 & 0.0 \% &0.0052 $\pm$ 0.0001 & 1569.0000 $\pm$ 0.0000 & 0.0562 $\pm$ 0.0005 & 17000.0000 $\pm$ 0.0000 \\
Q-learning ($\epsilon=0.25$) & 0.16348 $\pm$ 0.0017 & 48929.1500 $\pm$ 308.7902 & 0.0 \% &0.0069 $\pm$ 0.0019 & 2083.6000 $\pm$ 575.2788 & 0.0674 $\pm$ 0.0018 & 20340.0000 $\pm$ 533.2917 \\
Q-learning ($\epsilon=0.5$) & 0.24302 $\pm$ 0.0030 & 75008.7100 $\pm$ 538.2311 & 0.0 \% &0.0083 $\pm$ 0.0028 & 2588.2400 $\pm$ 895.4427 & 0.0801 $\pm$ 0.0022 & 24950.0000 $\pm$ 683.7397 \\
Q-learning ($\epsilon=0.75$) & 0.52179 $\pm$ 0.0074 & 166191.9100 $\pm$ 2022.8075 & 0.0 \% &0.0099 $\pm$ 0.0041 & 3199.6400 $\pm$ 1346.4165 & 0.1070 $\pm$ 0.0044 & 34520.0000 $\pm$ 1403.4244 \\
Q-learning ($\epsilon=0.9$) & 0.83570 $\pm$ 0.1904 & 272845.6200 $\pm$ 61283.9929 & 99.0 \% &0.0093 $\pm$ 0.0061 & 3071.7600 $\pm$ 2019.3443 & 0.1416 $\pm$ 0.0071 & 46750.0000 $\pm$ 2329.6996 \\
Q-learning ($\epsilon=1$) & 0.41468 $\pm$ 0.0594 & 144164.0500 $\pm$ 20481.0098 & 100.0 \% &0.0054 $\pm$ 0.0041 & 1859.9000 $\pm$ 1424.8237 & 0.2162 $\pm$ 0.0341 & 75080.0000 $\pm$ 11847.0925 \\
Q-learning\;($\pi$) ($\epsilon=1)$ & 0.42675 $\pm$ 0.0642 & 144673.6700 $\pm$ 21823.7557 & 100.0 \% &0.0058 $\pm$ 0.0055 & 1936.9800 $\pm$ 1860.8880 & 0.2262 $\pm$ 0.0405 & 76650.0000 $\pm$ 13750.9091 \\
\hline
\end{tabular}}
\end{table}
\vspace*{-1cm}
\begin{table}[!h]
\resizebox{0.75\paperwidth}{!}{
\hspace*{-7.7cm}\begin{tabular}{lccccccc}
\hline
Algorithm (Problem 12) & Run time (mean $\pm$ std) & \#actions (mean $\pm$ std) & Convergence & Discover goal time (mean $\pm$ std) & Discover goal \#actions (mean $\pm$ std) & Optimal Initial Cost2Go Time (mean $\pm$ std) & Optimal Initial Cost2Go \#actions (mean $\pm$ std) \\
\hline
Q-learning ($\epsilon=0$) & 0.39384 $\pm$ 0.0030 & 126594.0000 $\pm$ 0.0000 & 0.0 \% &0.0136 $\pm$ 0.0001 & 4334.0000 $\pm$ 0.0000 & 0.1777 $\pm$ 0.0013 & 57000.0000 $\pm$ 0.0000 \\
Q-learning ($\epsilon=0.25$) & 0.52039 $\pm$ 0.0040 & 164876.6200 $\pm$ 443.3426 & 0.0 \% &0.0131 $\pm$ 0.0074 & 4140.6400 $\pm$ 2339.7309 & 0.2097 $\pm$ 0.0035 & 66620.0000 $\pm$ 1017.6443 \\
Q-learning ($\epsilon=0.5$) & 0.74970 $\pm$ 0.0053 & 242867.3400 $\pm$ 945.7341 & 0.0 \% &0.0177 $\pm$ 0.0099 & 5750.2200 $\pm$ 3199.3502 & 0.2555 $\pm$ 0.0055 & 83200.0000 $\pm$ 1685.2300 \\
Q-learning ($\epsilon=0.75$) & 1.45247 $\pm$ 0.0116 & 485869.0000 $\pm$ 2962.9134 & 0.0 \% &0.0298 $\pm$ 0.0134 & 10033.1200 $\pm$ 4511.6419 & 0.3467 $\pm$ 0.0110 & 116750.0000 $\pm$ 3488.1944 \\
Q-learning ($\epsilon=0.9$) & 1.03468 $\pm$ 0.2863 & 354415.0400 $\pm$ 97640.2752 & 100.0 \% &0.0595 $\pm$ 0.0288 & 20555.3000 $\pm$ 9892.0178 & 0.4977 $\pm$ 0.0240 & 171190.0000 $\pm$ 7955.7464 \\
Q-learning ($\epsilon=1$) & 5.17925 $\pm$ 1.6084 & 1867336.2200 $\pm$ 579915.6964 & 94.0 \% &0.3157 $\pm$ 0.2704 & 114299.5000 $\pm$ 97747.7976 & 4.5908 $\pm$ 1.4548 & 1654969.3878 $\pm$ 523906.9954 \\
Q-learning\;($\pi$) ($\epsilon=1)$ & 5.26599 $\pm$ 1.0266 & 1793988.3800 $\pm$ 348867.4422 & 100.0 \% &0.1741 $\pm$ 0.1807 & 59990.7000 $\pm$ 62199.7318 & 5.1583 $\pm$ 1.0449 & 1757340.0000 $\pm$ 355406.9279 \\
\hline
\end{tabular}}
\end{table}
\vspace*{-1cm}
\begin{table}[!h]
\resizebox{0.75\paperwidth}{!}{
\hspace*{-7.7cm}\begin{tabular}{lccccccc}
\hline
Algorithm (Problem 13) & Run time (mean $\pm$ std) & \#actions (mean $\pm$ std) & Convergence & Discover goal time (mean $\pm$ std) & Discover goal \#actions (mean $\pm$ std) & Optimal Initial Cost2Go Time (mean $\pm$ std) & Optimal Initial Cost2Go \#actions (mean $\pm$ std) \\
\hline
Q-learning ($\epsilon=0$) & 0.15450 $\pm$ 0.0019 & 47576.0000 $\pm$ 0.0000 & 0.0 \% &0.0041 $\pm$ 0.0000 & 1260.0000 $\pm$ 0.0000 & 0.0649 $\pm$ 0.0007 & 20000.0000 $\pm$ 0.0000 \\
Q-learning ($\epsilon=0.25$) & 0.20974 $\pm$ 0.0027 & 63528.4200 $\pm$ 262.0945 & 0.0 \% &0.0084 $\pm$ 0.0023 & 2589.2200 $\pm$ 721.6648 & 0.0759 $\pm$ 0.0014 & 23110.0000 $\pm$ 371.3489 \\
Q-learning ($\epsilon=0.5$) & 0.30891 $\pm$ 0.0039 & 96395.7400 $\pm$ 547.0078 & 0.0 \% &0.0103 $\pm$ 0.0028 & 3261.4400 $\pm$ 901.9227 & 0.0911 $\pm$ 0.0022 & 28630.0000 $\pm$ 559.5534 \\
Q-learning ($\epsilon=0.75$) & 0.66353 $\pm$ 0.0088 & 214642.2200 $\pm$ 2092.3259 & 0.0 \% &0.0130 $\pm$ 0.0041 & 4265.2800 $\pm$ 1350.3880 & 0.1215 $\pm$ 0.0043 & 39680.0000 $\pm$ 1362.9380 \\
Q-learning ($\epsilon=0.9$) & 0.46075 $\pm$ 0.0709 & 152405.0000 $\pm$ 23204.2928 & 100.0 \% &0.0140 $\pm$ 0.0074 & 4669.0000 $\pm$ 2496.5186 & 0.1786 $\pm$ 0.0122 & 59420.0000 $\pm$ 4108.9658 \\
Q-learning ($\epsilon=1$) & 0.72476 $\pm$ 0.1851 & 251488.8400 $\pm$ 64321.8408 & 100.0 \% &0.0128 $\pm$ 0.0116 & 4440.9400 $\pm$ 4085.6970 & 0.4882 $\pm$ 0.1385 & 169640.0000 $\pm$ 48080.4576 \\
Q-learning\;($\pi$) ($\epsilon=1)$ & 0.82376 $\pm$ 0.2249 & 278237.9200 $\pm$ 75499.1909 & 100.0 \% &0.0109 $\pm$ 0.0109 & 3722.1400 $\pm$ 3760.4954 & 0.5312 $\pm$ 0.1547 & 179730.0000 $\pm$ 51951.8729 \\
\hline
\end{tabular}}
\end{table}
\vspace*{-1cm}
\begin{table}[!h]
\resizebox{0.75\paperwidth}{!}{
\hspace*{-7.7cm}\begin{tabular}{lccccccc}
\hline
Algorithm (Problem 14) & Run time (mean $\pm$ std) & \#actions (mean $\pm$ std) & Convergence & Discover goal time (mean $\pm$ std) & Discover goal \#actions (mean $\pm$ std) & Optimal Initial Cost2Go Time (mean $\pm$ std) & Optimal Initial Cost2Go \#actions (mean $\pm$ std) \\
\hline
Q-learning ($\epsilon=0$) & 0.11687 $\pm$ 0.0011 & 36008.0000 $\pm$ 0.0000 & 0.0 \% &0.0053 $\pm$ 0.0000 & 1615.0000 $\pm$ 0.0000 & 0.0517 $\pm$ 0.0005 & 16000.0000 $\pm$ 0.0000 \\
Q-learning ($\epsilon=0.25$) & 0.15673 $\pm$ 0.0016 & 48378.9900 $\pm$ 210.8762 & 0.0 \% &0.0078 $\pm$ 0.0023 & 2411.4400 $\pm$ 713.8361 & 0.0620 $\pm$ 0.0016 & 19250.0000 $\pm$ 433.0127 \\
Q-learning ($\epsilon=0.5$) & 0.22881 $\pm$ 0.0030 & 72961.0200 $\pm$ 601.7355 & 0.0 \% &0.0098 $\pm$ 0.0026 & 3142.1200 $\pm$ 846.0950 & 0.0729 $\pm$ 0.0020 & 23390.0000 $\pm$ 598.2474 \\
Q-learning ($\epsilon=0.75$) & 0.47538 $\pm$ 0.0067 & 156573.9700 $\pm$ 1738.3132 & 0.0 \% &0.0113 $\pm$ 0.0042 & 3729.0000 $\pm$ 1393.5140 & 0.0978 $\pm$ 0.0038 & 32440.0000 $\pm$ 1210.9500 \\
Q-learning ($\epsilon=0.9$) & 0.48236 $\pm$ 0.0984 & 163299.2600 $\pm$ 32972.2192 & 100.0 \% &0.0095 $\pm$ 0.0062 & 3224.5600 $\pm$ 2110.4774 & 0.1341 $\pm$ 0.0069 & 45590.0000 $\pm$ 2293.8832 \\
Q-learning ($\epsilon=1$) & 0.30029 $\pm$ 0.0322 & 106857.6900 $\pm$ 11388.7065 & 100.0 \% &0.0057 $\pm$ 0.0045 & 2018.8400 $\pm$ 1595.4904 & 0.1874 $\pm$ 0.0205 & 66620.0000 $\pm$ 7206.6358 \\
Q-learning\;($\pi$) ($\epsilon=1)$ & 0.31998 $\pm$ 0.0378 & 108522.8800 $\pm$ 12668.4291 & 100.0 \% &0.0054 $\pm$ 0.0050 & 1817.5800 $\pm$ 1713.0152 & 0.1942 $\pm$ 0.0222 & 65910.0000 $\pm$ 7482.1053 \\
\hline
\end{tabular}}
\end{table}
\vspace*{-1cm}
\begin{table}[!h]
\resizebox{0.75\paperwidth}{!}{
\hspace*{-7.7cm}\begin{tabular}{lccccccc}
\hline
Algorithm (Problem 15) & Run time (mean $\pm$ std) & \#actions (mean $\pm$ std) & Convergence & Discover goal time (mean $\pm$ std) & Discover goal \#actions (mean $\pm$ std) & Optimal Initial Cost2Go Time (mean $\pm$ std) & Optimal Initial Cost2Go \#actions (mean $\pm$ std) \\
\hline
Q-learning ($\epsilon=0$) & 0.11131 $\pm$ 0.0010 & 33372.0000 $\pm$ 0.0000 & 0.0 \% &0.0040 $\pm$ 0.0000 & 1207.0000 $\pm$ 0.0000 & 0.0599 $\pm$ 0.0005 & 18000.0000 $\pm$ 0.0000 \\
Q-learning ($\epsilon=0.25$) & 0.14548 $\pm$ 0.0015 & 43947.8600 $\pm$ 277.4724 & 0.0 \% &0.0071 $\pm$ 0.0016 & 2168.5400 $\pm$ 506.3488 & 0.0703 $\pm$ 0.0018 & 21410.0000 $\pm$ 511.7617 \\
Q-learning ($\epsilon=0.5$) & 0.20891 $\pm$ 0.0020 & 64789.3400 $\pm$ 537.2428 & 0.0 \% &0.0081 $\pm$ 0.0020 & 2559.7000 $\pm$ 650.5667 & 0.0855 $\pm$ 0.0029 & 26800.0000 $\pm$ 894.4272 \\
Q-learning ($\epsilon=0.75$) & 0.40536 $\pm$ 0.0057 & 130578.3200 $\pm$ 1350.2966 & 0.0 \% &0.0085 $\pm$ 0.0036 & 2766.3800 $\pm$ 1169.2164 & 0.1189 $\pm$ 0.0062 & 38750.0000 $\pm$ 1986.8316 \\
Q-learning ($\epsilon=0.9$) & 0.95809 $\pm$ 0.0194 & 315815.5100 $\pm$ 5887.7182 & 0.0 \% &0.0065 $\pm$ 0.0040 & 2155.7000 $\pm$ 1334.3122 & 0.1624 $\pm$ 0.0131 & 54110.0000 $\pm$ 4402.0336 \\
Q-learning ($\epsilon=1$) & 0.40327 $\pm$ 0.0676 & 141574.9600 $\pm$ 23520.4918 & 100.0 \% &0.0037 $\pm$ 0.0030 & 1279.2800 $\pm$ 1055.8500 & 0.2170 $\pm$ 0.0238 & 76250.0000 $\pm$ 8462.1215 \\
Q-learning\;($\pi$) ($\epsilon=1)$ & 0.39702 $\pm$ 0.0784 & 135803.2500 $\pm$ 26665.6822 & 100.0 \% &0.0029 $\pm$ 0.0027 & 994.5000 $\pm$ 942.6673 & 0.2230 $\pm$ 0.0282 & 76310.0000 $\pm$ 9585.0874 \\
\hline
\end{tabular}}
\end{table}
\vspace*{-1cm}
\begin{table}[!h]
\resizebox{0.75\paperwidth}{!}{
\hspace*{-7.7cm}\begin{tabular}{lccccccc}
\hline
Algorithm (Problem 16) & Run time (mean $\pm$ std) & \#actions (mean $\pm$ std) & Convergence & Discover goal time (mean $\pm$ std) & Discover goal \#actions (mean $\pm$ std) & Optimal Initial Cost2Go Time (mean $\pm$ std) & Optimal Initial Cost2Go \#actions (mean $\pm$ std) \\
\hline
Q-learning ($\epsilon=0$) & 0.00285 $\pm$ 0.0000 & 1000.0000 $\pm$ 0.0000 & 100.0 \% &0.0000 $\pm$ 0.0000 & 14.0000 $\pm$ 0.0000 & 0.0028 $\pm$ 0.0000 & 1000.0000 $\pm$ 0.0000 \\
Q-learning ($\epsilon=0.25$) & 0.00289 $\pm$ 0.0000 & 1002.1600 $\pm$ 2.1294 & 100.0 \% &0.0001 $\pm$ 0.0000 & 16.0000 $\pm$ 3.3941 & 0.0029 $\pm$ 0.0000 & 1000.0000 $\pm$ 0.0000 \\
Q-learning ($\epsilon=0.5$) & 0.00283 $\pm$ 0.0000 & 1004.0000 $\pm$ 3.7094 & 100.0 \% &0.0001 $\pm$ 0.0000 & 18.3000 $\pm$ 7.0050 & 0.0028 $\pm$ 0.0000 & 1000.0000 $\pm$ 0.0000 \\
Q-learning ($\epsilon=0.75$) & 0.00278 $\pm$ 0.0000 & 1006.4000 $\pm$ 6.9685 & 100.0 \% &0.0001 $\pm$ 0.0000 & 20.9800 $\pm$ 10.2781 & 0.0028 $\pm$ 0.0000 & 1000.0000 $\pm$ 0.0000 \\
Q-learning ($\epsilon=0.9$) & 0.00276 $\pm$ 0.0000 & 1013.2800 $\pm$ 14.3095 & 100.0 \% &0.0001 $\pm$ 0.0001 & 26.3400 $\pm$ 20.9496 & 0.0027 $\pm$ 0.0000 & 1000.0000 $\pm$ 0.0000 \\
Q-learning ($\epsilon=1$) & 0.00266 $\pm$ 0.0001 & 1019.6000 $\pm$ 22.4410 & 100.0 \% &0.0001 $\pm$ 0.0001 & 23.6000 $\pm$ 19.2603 & 0.0026 $\pm$ 0.0000 & 1000.0000 $\pm$ 0.0000 \\
Q-learning\;($\pi$) ($\epsilon=1)$ & 0.00307 $\pm$ 0.0001 & 1023.6400 $\pm$ 24.1386 & 100.0 \% &0.0001 $\pm$ 0.0001 & 21.6400 $\pm$ 22.2484 & 0.0030 $\pm$ 0.0000 & 1000.0000 $\pm$ 0.0000 \\
\hline
\end{tabular}}
\end{table}

\vspace*{-1cm}
\begin{table}[!h]
\label{tab:1_decay}
\caption{Algorithm performance on Problem 1, run on a desktop computer with an i5-12400F CPU (2.50GHz). Includes \ql{} with $\epsilon$-decay.}
\resizebox{0.75\paperwidth}{!}{
\hspace*{-7.7cm}\begin{tabular}{lccccccc}
\hline
Algorithm (Problem 1) & Runtime (mean $\pm$ std) & \#actions (mean $\pm$ std) & Convergence & Discover goal time (mean $\pm$ std) & Discover goal \#actions (mean $\pm$ std) & Optimal Initial Cost2Go Time (mean $\pm$ std) & Optimal Initial Cost2Go \#actions (mean $\pm$ std) \\
\hline
Q-learning ($\epsilon=0$) & 0.12396 $\pm$ 0.0077 & 40894.0000 $\pm$ 0.0000 & 0.0 \% &0.0033 $\pm$ 0.0009 & 1127.0000 $\pm$ 0.0000 & 0.0605 $\pm$ 0.0055 & 20000.0000 $\pm$ 0.0000 \\
Q-learning ($\epsilon=0.25$) & 0.16006 $\pm$ 0.0090 & 53322.1200 $\pm$ 224.6660 & 0.0 \% &0.0065 $\pm$ 0.0026 & 2143.6200 $\pm$ 770.4978 & 0.0707 $\pm$ 0.0074 & 23530.0000 $\pm$ 499.0992 \\
Q-learning ($\epsilon=0.5$) & 0.23343 $\pm$ 0.0113 & 79207.6000 $\pm$ 547.0679 & 0.0 \% &0.0093 $\pm$ 0.0037 & 3057.7400 $\pm$ 979.0199 & 0.0874 $\pm$ 0.0066 & 29590.0000 $\pm$ 722.4265 \\
Q-learning ($\epsilon=0.75$) & 0.49495 $\pm$ 0.0278 & 172365.5800 $\pm$ 7977.4984 & 10.0 \% &0.0106 $\pm$ 0.0043 & 3679.7600 $\pm$ 1235.4283 & 0.1206 $\pm$ 0.0092 & 42170.0000 $\pm$ 1697.3803 \\
Q-learning ($\epsilon=0.9$) & 0.25825 $\pm$ 0.0410 & 92243.3600 $\pm$ 14580.8509 & 100.0 \% &0.0092 $\pm$ 0.0053 & 3264.8600 $\pm$ 1844.4111 & 0.1650 $\pm$ 0.0133 & 58850.0000 $\pm$ 3561.9517 \\
Q-learning ($\epsilon=1$) & 0.30633 $\pm$ 0.0452 & 114753.3800 $\pm$ 16236.2484 & 100.0 \% &0.0039 $\pm$ 0.0038 & 1413.7000 $\pm$ 1342.0072 & 0.2580 $\pm$ 0.0400 & 96400.0000 $\pm$ 14849.2424 \\
Q-learning\;($\pi$) ($\epsilon=1)$ & 0.30954 $\pm$ 0.0371 & 114637.3700 $\pm$ 13326.4214 & 100.0 \% &0.0046 $\pm$ 0.0039 & 1643.4800 $\pm$ 1398.7443 & 0.2630 $\pm$ 0.0358 & 97310.0000 $\pm$ 12780.2152 \\
\hline
\end{tabular}}
\end{table}

\vspace*{-1cm}
\begin{table}[!h]
\label{tab:10_decay}
\caption{Algorithm performance on Problem 10, run on a desktop computer with an i5-12400F CPU (2.50GHz). Includes \ql{} with $\epsilon$-decay.}
\resizebox{0.75\paperwidth}{!}{
\hspace*{-7.7cm}\begin{tabular}{lccccccc}
\hline
Algorithm (Problem 10) & Runtime (mean $\pm$ std) & \#actions (mean $\pm$ std) & Convergence & Discover goal time (mean $\pm$ std) & Discover goal \#actions (mean $\pm$ std) & Optimal Initial Cost2Go Time (mean $\pm$ std) & Optimal Initial Cost2Go \#actions (mean $\pm$ std) \\
\hline
Q-learning ($\epsilon=0$) & 0.29829 $\pm$ 0.0188 & 101214.0000 $\pm$ 0.0000 & 0.0 \% &0.0061 $\pm$ 0.0013 & 2015.0000 $\pm$ 0.0000 & 0.1530 $\pm$ 0.0123 & 52000.0000 $\pm$ 0.0000 \\
Q-learning ($\epsilon=0.25$) & 0.39699 $\pm$ 0.0211 & 133139.2000 $\pm$ 427.4064 & 0.0 \% &0.0089 $\pm$ 0.0049 & 3046.7600 $\pm$ 1631.0804 & 0.1848 $\pm$ 0.0117 & 61960.0000 $\pm$ 1085.5413 \\
Q-learning ($\epsilon=0.5$) & 0.57380 $\pm$ 0.0171 & 200169.4200 $\pm$ 1030.5368 & 0.0 \% &0.0141 $\pm$ 0.0065 & 4878.0600 $\pm$ 2268.3855 & 0.2258 $\pm$ 0.0121 & 78600.0000 $\pm$ 1870.8287 \\
Q-learning ($\epsilon=0.75$) & 0.92641 $\pm$ 0.1920 & 333304.1300 $\pm$ 68721.9494 & 83.0 \% &0.0210 $\pm$ 0.0076 & 7685.2200 $\pm$ 2637.8776 & 0.3114 $\pm$ 0.0163 & 112850.0000 $\pm$ 3612.1323 \\
Q-learning ($\epsilon=0.9$) & 0.54554 $\pm$ 0.0515 & 203946.4000 $\pm$ 18464.7784 & 100.0 \% &0.0333 $\pm$ 0.0180 & 12237.4000 $\pm$ 6515.6084 & 0.4278 $\pm$ 0.0261 & 160460.0000 $\pm$ 7338.1469 \\
Q-learning ($\epsilon=1$) & 1.69006 $\pm$ 0.5458 & 658601.9700 $\pm$ 211027.5533 & 100.0 \% &0.0649 $\pm$ 0.0751 & 25902.0000 $\pm$ 29790.2904 & 1.1196 $\pm$ 0.2805 & 437430.0000 $\pm$ 108380.5568 \\
Q-learning\;($\pi$) ($\epsilon=1)$ & 1.78199 $\pm$ 0.4455 & 685022.7200 $\pm$ 169670.6191 & 100.0 \% &0.0472 $\pm$ 0.0505 & 18366.3400 $\pm$ 19499.6667 & 1.1819 $\pm$ 0.2485 & 455760.0000 $\pm$ 94888.6843 \\
\hline
\end{tabular}}
\end{table}

\vspace*{-1cm}
\begin{table}[!h]
\label{tab:11_decay}
\caption{Algorithm performance on Problem 11, run on a desktop computer with an i5-12400F CPU (2.50GHz). Includes \ql{} with $\epsilon$-decay.}
\resizebox{0.75\paperwidth}{!}{
\hspace*{-7.7cm}\begin{tabular}{lccccccc}
\hline
Algorithm (Problem 11) & Runtime (mean $\pm$ std) & \#actions (mean $\pm$ std) & Convergence & Discover goal time (mean $\pm$ std) & Discover goal \#actions (mean $\pm$ std) & Optimal Initial Cost2Go Time (mean $\pm$ std) & Optimal Initial Cost2Go \#actions (mean $\pm$ std) \\
\hline
Q-learning ($\epsilon=0$) & 0.10471 $\pm$ 0.0075 & 35316.0000 $\pm$ 0.0000 & 0.0 \% &0.0047 $\pm$ 0.0011 & 1569.0000 $\pm$ 0.0000 & 0.0501 $\pm$ 0.0039 & 17000.0000 $\pm$ 0.0000 \\
Q-learning ($\epsilon=0.25$) & 0.14444 $\pm$ 0.0083 & 48838.7600 $\pm$ 311.0242 & 0.0 \% &0.0061 $\pm$ 0.0019 & 2081.9200 $\pm$ 537.3621 & 0.0601 $\pm$ 0.0056 & 20320.0000 $\pm$ 487.4423 \\
Q-learning ($\epsilon=0.5$) & 0.21523 $\pm$ 0.0092 & 74975.5100 $\pm$ 580.9994 & 0.0 \% &0.0073 $\pm$ 0.0029 & 2529.5200 $\pm$ 889.7264 & 0.0712 $\pm$ 0.0061 & 24870.0000 $\pm$ 770.1299 \\
Q-learning ($\epsilon=0.75$) & 0.46522 $\pm$ 0.0216 & 166059.1000 $\pm$ 2033.6388 & 0.0 \% &0.0095 $\pm$ 0.0046 & 3428.7800 $\pm$ 1434.6559 & 0.0961 $\pm$ 0.0076 & 34520.0000 $\pm$ 1590.4716 \\
Q-learning ($\epsilon=0.9$) & 0.73656 $\pm$ 0.1717 & 270569.9300 $\pm$ 60522.7662 & 98.0 \% &0.0085 $\pm$ 0.0055 & 3121.8000 $\pm$ 2038.0453 & 0.1253 $\pm$ 0.0110 & 46730.0000 $\pm$ 2918.4071 \\
Q-learning ($\epsilon=1$) & 0.36355 $\pm$ 0.0525 & 141925.4500 $\pm$ 17284.5134 & 100.0 \% &0.0055 $\pm$ 0.0051 & 2234.1800 $\pm$ 2136.7621 & 0.1997 $\pm$ 0.0398 & 77680.0000 $\pm$ 14925.7362 \\
Q-learning\;($\pi$) ($\epsilon=1)$ & 0.37555 $\pm$ 0.0586 & 144673.6700 $\pm$ 21823.7557 & 100.0 \% &0.0051 $\pm$ 0.0051 & 1936.9800 $\pm$ 1860.8880 & 0.1989 $\pm$ 0.0351 & 76650.0000 $\pm$ 13750.9091 \\
\hline
\end{tabular}}
\end{table}

\section{Figures for stochastic experiments}
\begin{figure}[h!]
    \centering
        \hspace*{-3cm}\includegraphics[width=1.5\linewidth]{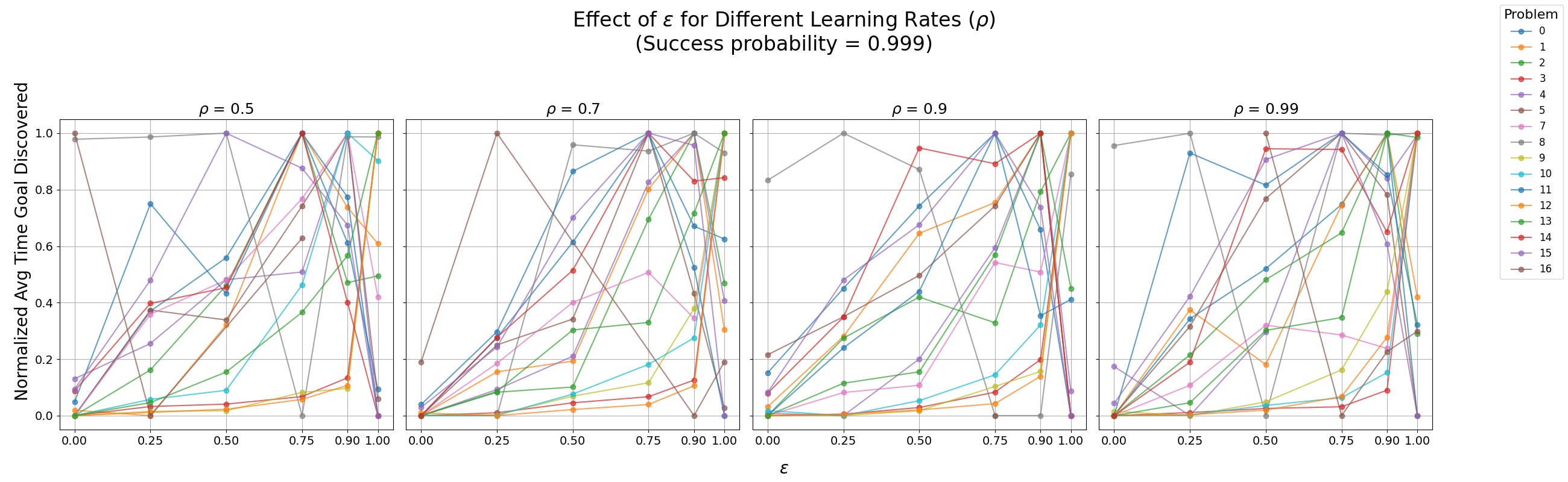}
    \caption{The effect of decreasing greediness in a policy when applied on different problems with $\gamma = 0.999$ and varying learning rate $\rho$.}
\end{figure}

\begin{figure}[h!]
    \centering
        \hspace*{-3cm}\includegraphics[width=1.5\linewidth]{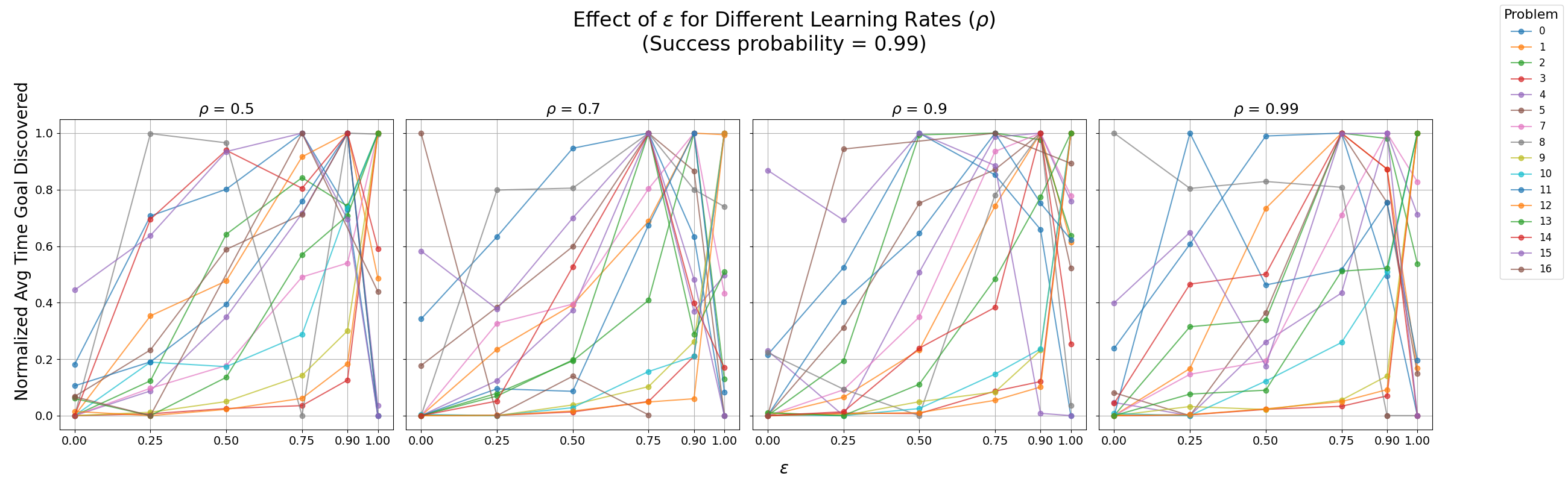}
    \caption{The effect of decreasing greediness in a policy when applied on different problems with $\gamma = 0.99$ and varying learning rate $\rho$.}
\end{figure}

\begin{figure}[h]
    \centering
    \includegraphics[width=\linewidth]{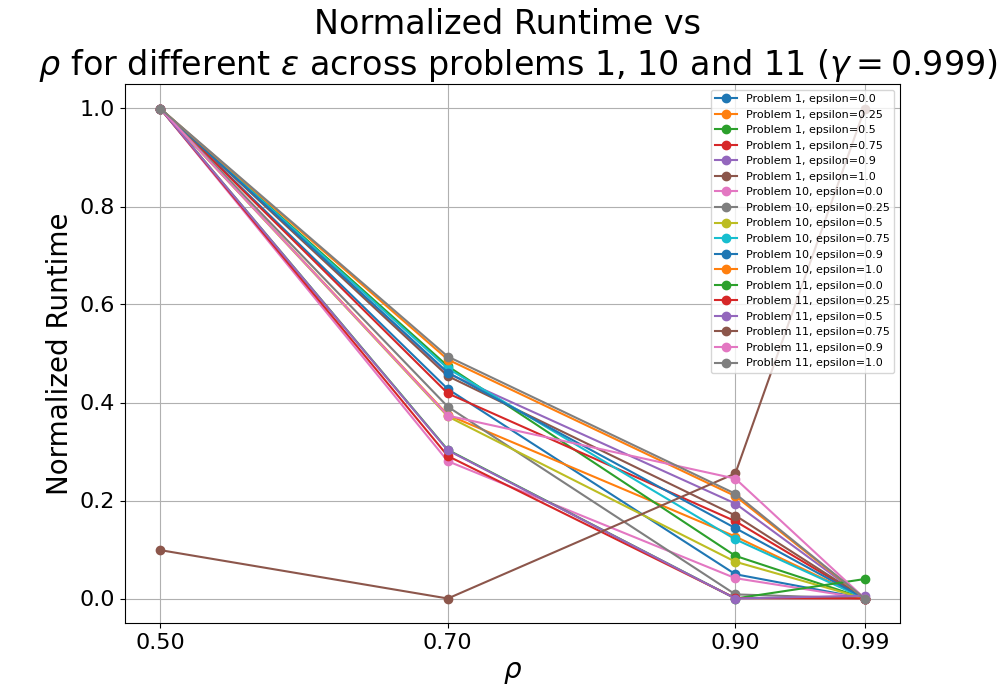}
    \caption{As $\rho$ increases, the run time for \ql{} decreases when $\gamma = 0.999$ for all $\epsilon$ with the exception of Problem 11 with $\epsilon = 0.75$.}
\end{figure}

\section{Stochastic experiments}
Bellow we have the tables for Problems 1,10 and 11 with $\gamma  \in\{ 0.99, 0.99\}$.
These algorithms were run $N = 100$ times.
The table shows similar metrics as in the deterministic case, but we also list how often the optimal cost-to-go value for the initial state is achieved, and the shortest and longest paths extracted from the computed values. 
Afterwards, the tables for Problems 1,10 and 11 with $\gamma \in\{ 0.9, 0.7, 0.6\}$ are shown.
These were run $N = 10$ times since they were considerably slower due to the stochasticity.
We also used a desktop computer with an i5-12400F CPU (2.50GHz) for this batch of experiments and the ones we describe below.
We show results using $\epsilon$-decay in Tables 22, 23 and 24.
Additional experiments with low $\gamma$ for problems 8 and 16 are also included.
Finally, the tables for the rest of the problems ran with $\gamma \in \{0.99, 0.999\}$ are shown.
\begin{table}[!h]
\caption{Performance of DP methods and \ql{} as we change the learning rate $\rho$ in stochastic problem with $\gamma = 0.99$ (Problem 1).}
\label{tab:prob1stoch0.99}
\resizebox{\textwidth}{!}{
\begin{tabular}{lccccccccc}
\hline
\makecell{Algorithm (Problem 1) \\ ($\gamma$ = 0.99)} & \makecell{Run time \\ (mean $\pm$ std)}& \makecell{\#actions \\ (mean $\pm$ std)} & Convergence & \makecell{Discover goal \\ time \\ (mean $\pm$ std)} & \makecell{Discover goal \\ \#actions \\ (mean $\pm$ std) }& \makecell{Optimal Initial \\ Cost2Go Time \\ (mean $\pm$ std)} & \makecell{Optimal Initial \\ Cost2Go \#actions \\ (mean $\pm$ std)}& \makecell{Optimal Initial \\ Cost2Go \\ Convergence}& \makecell{Shortest/Longest \\ Path} \\
\hline
Q-learning ($\rho=0.5, \epsilon=0$) & 0.44775 $\pm$ 0.0094 & 110018.9500 $\pm$ 120.6576 & 0.0\% & 0.0041 $\pm$ 0.0013 & 1016.9000 $\pm$ 309.0232 & 0.1815 $\pm$ 0.0082 & 44680.0000 $\pm$ 1907.2493 & 100.0\% & 28/34 \\
Q-learning ($\rho=0.7, \epsilon=0$) & 0.42652 $\pm$ 0.0104 & 103655.8100 $\pm$ 848.3384 & 0.0\% & 0.0040 $\pm$ 0.0012 & 968.4300 $\pm$ 267.7274 & 0.1359 $\pm$ 0.0100 & 33040.0000 $\pm$ 2262.3881 & 100.0\% & 28/1456 \\
Q-learning ($\rho=0.9, \epsilon=0$) & 0.41722 $\pm$ 0.0076 & 102035.4600 $\pm$ 1339.1680 & 0.0\% & 0.0042 $\pm$ 0.0013 & 1014.5900 $\pm$ 292.7621 & 0.1306 $\pm$ 0.0429 & 31970.0000 $\pm$ 10627.7514 & 100.0\% & 28/920 \\
Q-learning ($\rho=0.99, \epsilon=0$) & 0.42034 $\pm$ 0.0090 & 102414.8500 $\pm$ 1556.8528 & 0.0\% & 0.0042 $\pm$ 0.0012 & 1026.1100 $\pm$ 282.5215 & 0.2078 $\pm$ 0.0750 & 50697.6744 $\pm$ 18153.7809 & 86.0\% & 28/2205 \\
Q-learning ($\rho=0.5, \epsilon=0.25$) & 0.56923 $\pm$ 0.0061 & 141115.7100 $\pm$ 320.6419 & 0.0\% & 0.0059 $\pm$ 0.0015 & 1472.4300 $\pm$ 377.8420 & 0.3409 $\pm$ 0.0901 & 84917.6471 $\pm$ 22314.9708 & 85.0\% & 28/32 \\
Q-learning ($\rho=0.7, \epsilon=0.25$) & 0.54557 $\pm$ 0.0069 & 134211.4300 $\pm$ 390.0833 & 0.0\% & 0.0050 $\pm$ 0.0014 & 1222.1600 $\pm$ 355.0452 & 0.3071 $\pm$ 0.1184 & 75800.0000 $\pm$ 28881.8282 & 50.0\% & 28/31 \\
\hline
\end{tabular}}
\end{table}

\begin{table}[!h]
\caption{Performance of DP methods and  \ql{} as we change the learning rate $\rho$ in stochastic problem with $\gamma = 0.999$ (Problem 1).}
\label{tab:prob1stoch0.999}
\resizebox{\textwidth}{!}{
\begin{tabular}{lccccccccc}
\hline
\makecell{Algorithm (Problem 1) \\ ($\gamma$ = 0.999)} & \makecell{Run time \\ (mean $\pm$ std)}& \makecell{\#actions \\ (mean $\pm$ std)} & Convergence & \makecell{Discover goal \\ time \\ (mean $\pm$ std)} & \makecell{Discover goal \\ \#actions \\ (mean $\pm$ std) }& \makecell{Optimal Initial \\ Cost2Go Time \\ (mean $\pm$ std)} & \makecell{Optimal Initial \\ Cost2Go \#actions \\ (mean $\pm$ std)}& \makecell{Optimal Initial \\ Cost2Go \\ Convergence}& \makecell{Shortest/Longest \\ Path} \\
\hline
Q-learning ($\rho=0.5, \epsilon=0$) & 0.45780 $\pm$ 0.0072 & 108128.1300 $\pm$ 66.4493 & 0.0\% & 0.0046 $\pm$ 0.0004 & 1072.3100 $\pm$ 94.0911 & 0.1763 $\pm$ 0.0038 & 41650.0000 $\pm$ 476.9696 & 100.0\% & 28/7264 \\
Q-learning ($\rho=0.7, \epsilon=0$) & 0.42858 $\pm$ 0.0087 & 100704.8600 $\pm$ 108.2371 & 0.0\% & 0.0045 $\pm$ 0.0005 & 1053.1400 $\pm$ 123.3148 & 0.1273 $\pm$ 0.0027 & 29950.0000 $\pm$ 217.9449 & 100.0\% & 28/1456 \\
Q-learning ($\rho=0.9, \epsilon=0$) & 0.40930 $\pm$ 0.0046 & 96842.1900 $\pm$ 116.4245 & 0.0\% & 0.0044 $\pm$ 0.0006 & 1041.3000 $\pm$ 148.1313 & 0.0973 $\pm$ 0.0012 & 23020.0000 $\pm$ 140.0000 & 100.0\% & 28/30 \\
Q-learning ($\rho=0.99, \epsilon=0$) & 0.40677 $\pm$ 0.0041 & 96121.0500 $\pm$ 195.6695 & 0.0\% & 0.0044 $\pm$ 0.0006 & 1044.4500 $\pm$ 130.2082 & 0.0888 $\pm$ 0.0008 & 21000.0000 $\pm$ 0.0000 & 100.0\% & 28/32 \\
Q-learning ($\rho=0.5, \epsilon=0.25$) & 0.57845 $\pm$ 0.0051 & 138790.8900 $\pm$ 386.8301 & 0.0\% & 0.0061 $\pm$ 0.0017 & 1468.7500 $\pm$ 399.6031 & 0.1898 $\pm$ 0.0019 & 45920.0000 $\pm$ 305.9412 & 100.0\% & 28/30 \\
Q-learning ($\rho=0.7, \epsilon=0.25$) & 0.54797 $\pm$ 0.0058 & 131788.1800 $\pm$ 402.8578 & 0.0\% & 0.0050 $\pm$ 0.0014 & 1217.3300 $\pm$ 349.6271 & 0.1381 $\pm$ 0.0026 & 33490.0000 $\pm$ 519.5190 & 100.0\% & 28/30 \\
\hline
\end{tabular}}
\end{table}

\begin{table}[!h]
\caption{Performance of DP methods and  \ql{} as we change the learning rate $\rho$ in stochastic problem with $\gamma = 0.99$ (Problem 10).}
\label{tab:prob10stoch0.99}
\resizebox{\textwidth}{!}{
\begin{tabular}{lccccccccc}
\hline
\makecell{Algorithm (Problem 10) \\ ($\gamma$ = 0.99)} & \makecell{Run time \\ (mean $\pm$ std)}& \makecell{\#actions \\ (mean $\pm$ std)} & Convergence & \makecell{Discover goal \\ time \\ (mean $\pm$ std)} & \makecell{Discover goal \\ \#actions \\ (mean $\pm$ std) }& \makecell{Optimal Initial \\ Cost2Go Time \\ (mean $\pm$ std)} & \makecell{Optimal Initial \\ Cost2Go \#actions \\ (mean $\pm$ std)}& \makecell{Optimal Initial \\ Cost2Go \\ Convergence}& \makecell{Shortest/Longest \\ Path} \\
\hline
Q-learning ($\rho=0.5, \epsilon=0$) & 1.09601 $\pm$ 0.0226 & 270673.0900 $\pm$ 1030.1652 & 0.0\% & 0.0098 $\pm$ 0.0041 & 2379.9400 $\pm$ 996.5445 & 0.4524 $\pm$ 0.0133 & 110860.0000 $\pm$ 2561.3278 & 100.0\% & 64/322 \\
Q-learning ($\rho=0.7, \epsilon=0$) & 1.03734 $\pm$ 0.0188 & 256745.5500 $\pm$ 3275.4224 & 0.0\% & 0.0099 $\pm$ 0.0039 & 2416.5100 $\pm$ 938.7294 & 0.3480 $\pm$ 0.1000 & 85742.2680 $\pm$ 25179.4706 & 97.0\% & 64/1614 \\
Q-learning ($\rho=0.9, \epsilon=0$) & 1.05040 $\pm$ 0.0185 & 259065.8600 $\pm$ 3638.1183 & 0.0\% & 0.0092 $\pm$ 0.0033 & 2244.6800 $\pm$ 804.8872 & 0.3135 $\pm$ 0.1367 & 77466.6667 $\pm$ 33639.9894 & 60.0\% & 66/45219 \\
Q-learning ($\rho=0.99, \epsilon=0$) & 1.07131 $\pm$ 0.0178 & 262565.8200 $\pm$ 3765.1413 & 0.0\% & 0.0093 $\pm$ 0.0036 & 2259.5400 $\pm$ 896.9925 & 0.2651 $\pm$ 0.0000 & 65000.0000 $\pm$ 0.0000 & 1.0\% & 66/2596 \\
Q-learning ($\rho=0.5, \epsilon=0.25$) & 1.43981 $\pm$ 0.0548 & 354087.5500 $\pm$ 622.1560 & 0.0\% & 0.0140 $\pm$ 0.0077 & 3431.0800 $\pm$ 1861.5218 & 0.6260 $\pm$ 0.1351 & 154170.0000 $\pm$ 32728.2920 & 100.0\% & 64/70 \\
Q-learning ($\rho=0.7, \epsilon=0.25$) & 1.34549 $\pm$ 0.0105 & 334650.5700 $\pm$ 1071.9651 & 0.0\% & 0.0105 $\pm$ 0.0067 & 2620.3600 $\pm$ 1668.5465 & 0.4836 $\pm$ 0.1409 & 120480.0000 $\pm$ 35038.4018 & 100.0\% & 64/112 \\
\hline
\end{tabular}}
\end{table}

\begin{table}[!h]
\caption{Performance of DP methods and  \ql{} as we change the learning rate $\rho$ in stochastic problem with $\gamma = 0.999$ (Problem 10).}
\label{tab:prob10stoch0.999}
\resizebox{\textwidth}{!}{
\begin{tabular}{lccccccccc}
\hline
\makecell{Algorithm (Problem 10) \\ ($\gamma$ = 0.999)} & \makecell{Run time \\ (mean $\pm$ std)}& \makecell{\#actions \\ (mean $\pm$ std)} & Convergence & \makecell{Discover goal \\ time \\ (mean $\pm$ std)} & \makecell{Discover goal \\ \#actions \\ (mean $\pm$ std) }& \makecell{Optimal Initial \\ Cost2Go Time \\ (mean $\pm$ std)} & \makecell{Optimal Initial \\ Cost2Go \#actions \\ (mean $\pm$ std)}& \makecell{Optimal Initial \\ Cost2Go \\ Convergence}& \makecell{Shortest/Longest \\ Path} \\
\hline
Q-learning ($\rho=0.5, \epsilon=0$) & 1.06396 $\pm$ 0.0155 & 265867.3800 $\pm$ 75.5957 & 0.0\% & 0.0118 $\pm$ 0.0029 & 2910.3800 $\pm$ 728.7117 & 0.4371 $\pm$ 0.0062 & 108280.0000 $\pm$ 601.3319 & 100.0\% & 64/48894 \\
Q-learning ($\rho=0.7, \epsilon=0$) & 0.97940 $\pm$ 0.0126 & 245376.2300 $\pm$ 1468.6347 & 0.0\% & 0.0101 $\pm$ 0.0016 & 2489.1400 $\pm$ 385.4470 & 0.3118 $\pm$ 0.0046 & 77440.0000 $\pm$ 682.9348 & 100.0\% & 64/8716 \\
Q-learning ($\rho=0.9, \epsilon=0$) & 0.95135 $\pm$ 0.0260 & 237506.9900 $\pm$ 2795.6508 & 0.0\% & 0.0102 $\pm$ 0.0015 & 2513.4300 $\pm$ 374.4821 & 0.2440 $\pm$ 0.0167 & 60840.0000 $\pm$ 4251.3998 & 100.0\% & 64/100002 \\
Q-learning ($\rho=0.99, \epsilon=0$) & 0.94648 $\pm$ 0.0174 & 236133.6700 $\pm$ 2755.9321 & 0.0\% & 0.0099 $\pm$ 0.0017 & 2432.8900 $\pm$ 408.7694 & 0.2290 $\pm$ 0.0314 & 57000.0000 $\pm$ 7900.6329 & 100.0\% & 64/27731 \\
Q-learning ($\rho=0.5, \epsilon=0.25$) & 1.39075 $\pm$ 0.0168 & 348458.2300 $\pm$ 550.3162 & 0.0\% & 0.0113 $\pm$ 0.0067 & 2844.8200 $\pm$ 1667.9490 & 0.4796 $\pm$ 0.0074 & 120530.0000 $\pm$ 921.4662 & 100.0\% & 64/66 \\
Q-learning ($\rho=0.7, \epsilon=0.25$) & 1.30497 $\pm$ 0.0156 & 328004.3100 $\pm$ 553.7289 & 0.0\% & 0.0109 $\pm$ 0.0067 & 2743.1800 $\pm$ 1696.3778 & 0.3480 $\pm$ 0.0050 & 87900.0000 $\pm$ 754.9834 & 100.0\% & 64/66 \\
\hline
\end{tabular}}
\end{table}

\begin{table}[!h]
\caption{Performance of DP methods and  \ql{} as we change the learning rate $\rho$ in stochastic problem with $\gamma = 0.99$ (Problem 11).}
\label{tab:prob11stoch0.99}
\resizebox{\textwidth}{!}{
\begin{tabular}{lccccccccc}
\hline
\makecell{Algorithm (Problem 11) \\ ($\gamma$ = 0.99)} & \makecell{Run time \\ (mean $\pm$ std)}& \makecell{\#actions \\ (mean $\pm$ std)} & Convergence & \makecell{Discover goal \\ time \\ (mean $\pm$ std)} & \makecell{Discover goal \\ \#actions \\ (mean $\pm$ std) }& \makecell{Optimal Initial \\ Cost2Go Time \\ (mean $\pm$ std)} & \makecell{Optimal Initial \\ Cost2Go \#actions \\ (mean $\pm$ std)}& \makecell{Optimal Initial \\ Cost2Go \\ Convergence}& \makecell{Shortest/Longest \\ Path} \\
\hline
Q-learning ($\rho=0.5, \epsilon=0$) & 0.39799 $\pm$ 0.0210 & 95423.6800 $\pm$ 567.6079 & 0.0\% & 0.0067 $\pm$ 0.0011 & 1633.4300 $\pm$ 258.5025 & 0.1518 $\pm$ 0.0101 & 36850.0000 $\pm$ 2321.0989 & 100.0\% & 22/100002 \\
Q-learning ($\rho=0.7, \epsilon=0$) & 0.37312 $\pm$ 0.0115 & 90568.8000 $\pm$ 2163.3390 & 0.0\% & 0.0067 $\pm$ 0.0011 & 1637.7600 $\pm$ 266.1527 & 0.1241 $\pm$ 0.0323 & 30404.0404 $\pm$ 7841.7675 & 99.0\% & 22/310 \\
Q-learning ($\rho=0.9, \epsilon=0$) & 0.38073 $\pm$ 0.0124 & 92259.2500 $\pm$ 2742.0887 & 0.0\% & 0.0068 $\pm$ 0.0010 & 1675.1000 $\pm$ 256.6727 & 0.1323 $\pm$ 0.0711 & 32383.5616 $\pm$ 17337.6067 & 73.0\% & 22/3046 \\
Q-learning ($\rho=0.99, \epsilon=0$) & 0.38594 $\pm$ 0.0130 & 93504.7500 $\pm$ 2852.7542 & 0.0\% & 0.0065 $\pm$ 0.0008 & 1590.2000 $\pm$ 183.6596 & 0.1885 $\pm$ 0.0765 & 45947.3684 $\pm$ 18453.2801 & 38.0\% & 22/696 \\
Q-learning ($\rho=0.5, \epsilon=0.25$) & 0.51791 $\pm$ 0.0067 & 126854.4000 $\pm$ 347.8654 & 0.0\% & 0.0080 $\pm$ 0.0018 & 1994.0800 $\pm$ 454.6916 & 0.1593 $\pm$ 0.0079 & 39700.0000 $\pm$ 1920.9373 & 100.0\% & 22/27 \\
Q-learning ($\rho=0.7, \epsilon=0.25$) & 0.49021 $\pm$ 0.0080 & 119017.3200 $\pm$ 575.3976 & 0.0\% & 0.0074 $\pm$ 0.0015 & 1833.9600 $\pm$ 376.3915 & 0.1308 $\pm$ 0.0220 & 32300.0000 $\pm$ 5243.0907 & 100.0\% & 22/26 \\
\hline
\end{tabular}}
\end{table}

\begin{table}[!h]
\caption{Performance of DP methods and  \ql{} as we change the learning rate $\rho$ in stochastic problem with $\gamma = 0.999$ (Problem 11).}
\label{tab:prob11stoch0.999}
\resizebox{\textwidth}{!}{
\begin{tabular}{lccccccccc}
\hline
\makecell{Algorithm (Problem 11) \\ ($\gamma$ = 0.999)} & \makecell{Run time \\ (mean $\pm$ std)}& \makecell{\#actions \\ (mean $\pm$ std)} & Convergence & \makecell{Discover goal \\ time \\ (mean $\pm$ std)} & \makecell{Discover goal \\ \#actions \\ (mean $\pm$ std) }& \makecell{Optimal Initial \\ Cost2Go Time \\ (mean $\pm$ std)} & \makecell{Optimal Initial \\ Cost2Go \#actions \\ (mean $\pm$ std)}& \makecell{Optimal Initial \\ Cost2Go \\ Convergence}& \makecell{Shortest/Longest \\ Path} \\
\hline
Q-learning ($\rho=0.5, \epsilon=0$) & 0.38377 $\pm$ 0.0045 & 93082.8800 $\pm$ 193.1769 & 0.0\% & 0.0056 $\pm$ 0.0004 & 1375.0100 $\pm$ 86.5073 & 0.1436 $\pm$ 0.0018 & 35020.0000 $\pm$ 140.0000 & 100.0\% & 22/26 \\
Q-learning ($\rho=0.7, \epsilon=0$) & 0.35480 $\pm$ 0.0113 & 85400.1200 $\pm$ 660.9158 & 0.0\% & 0.0056 $\pm$ 0.0005 & 1372.6300 $\pm$ 108.3105 & 0.1035 $\pm$ 0.0017 & 25130.0000 $\pm$ 364.8287 & 100.0\% & 22/46025 \\
Q-learning ($\rho=0.9, \epsilon=0$) & 0.34218 $\pm$ 0.0090 & 82557.8800 $\pm$ 1606.2870 & 0.0\% & 0.0056 $\pm$ 0.0003 & 1357.6200 $\pm$ 65.3582 & 0.0828 $\pm$ 0.0021 & 20160.0000 $\pm$ 417.6123 & 100.0\% & 22/21524 \\
Q-learning ($\rho=0.99, \epsilon=0$) & 0.34383 $\pm$ 0.0101 & 82572.4900 $\pm$ 1944.7986 & 0.0\% & 0.0056 $\pm$ 0.0003 & 1360.0700 $\pm$ 77.6709 & 0.0793 $\pm$ 0.0028 & 19240.0000 $\pm$ 649.9231 & 100.0\% & 22/8643 \\
Q-learning ($\rho=0.5, \epsilon=0.25$) & 0.51539 $\pm$ 0.0073 & 124786.3700 $\pm$ 373.1031 & 0.0\% & 0.0080 $\pm$ 0.0019 & 1971.3500 $\pm$ 464.3022 & 0.1532 $\pm$ 0.0029 & 37780.0000 $\pm$ 414.2463 & 100.0\% & 22/24 \\
Q-learning ($\rho=0.7, \epsilon=0.25$) & 0.47911 $\pm$ 0.0046 & 116563.4500 $\pm$ 358.9886 & 0.0\% & 0.0075 $\pm$ 0.0016 & 1861.1700 $\pm$ 391.1360 & 0.1129 $\pm$ 0.0013 & 28010.0000 $\pm$ 223.3831 & 100.0\% & 22/24 \\
\hline
\end{tabular}}
\end{table}

\begin{table}[!h]
\caption{Performance of DP methods and  \ql{} as we change the learning rate $\rho$ in stochastic problem with $\gamma = 0.5$ (Problem 1).}
\label{tab:prob1stoch0.5}
\resizebox{\textwidth}{!}{
\begin{tabular}{lccccccccc}
\hline
\makecell{Algorithm (Problem 1) \\ ($\gamma$ = 0.5)} & \makecell{Run time \\ (mean $\pm$ std)}& \makecell{\#actions \\ (mean $\pm$ std)} & Convergence & \makecell{Discover goal \\ time \\ (mean $\pm$ std)} & \makecell{Discover goal \\ \#actions \\ (mean $\pm$ std) }& \makecell{Optimal Initial 
\\ Cost2Go Time \\ (mean $\pm$ std)} & \makecell{Optimal Initial \\ Cost2Go \#actions \\ (mean $\pm$ std)}& \makecell{Optimal Initial \\ Cost2Go \\ Convergence}& \makecell{Shortest/Longest \\ Path} \\
\hline
Q-learning ($\rho=0.05, \epsilon=0$) & 2.82370 $\pm$ 0.0630 & 829360.5000 $\pm$ 3538.3784 & 0.0\% & 0.0077 $\pm$ 0.0039 & 2314.0000 $\pm$ 1165.9239 & nan $\pm$ nan & nan $\pm$ nan & 0.0\% & 66/226 \\
Q-learning ($\rho=0.1, \epsilon=0$) & 1.93985 $\pm$ 0.0274 & 571004.0000 $\pm$ 4104.9053 & 0.0\% & 0.0092 $\pm$ 0.0051 & 2723.5000 $\pm$ 1523.5988 & 1.9076 $\pm$ 0.0383 & 562833.3333 $\pm$ 5698.4403 & 60.0\% & 67/152 \\
Q-learning ($\rho=0.3, \epsilon=0$) & 1.22874 $\pm$ 0.0288 & 362432.8000 $\pm$ 4119.4320 & 0.0\% & 0.0067 $\pm$ 0.0020 & 1878.2000 $\pm$ 560.1846 & 0.6407 $\pm$ 0.0242 & 189000.0000 $\pm$ 7886.6977 & 100.0\% & 50/99 \\
Q-learning ($\rho=0.5, \epsilon=0$) & 1.10654 $\pm$ 0.0330 & 326328.2000 $\pm$ 7451.1840 & 0.0\% & 0.0078 $\pm$ 0.0023 & 2399.8000 $\pm$ 831.8626 & 0.5686 $\pm$ 0.1974 & 169200.0000 $\pm$ 58245.6865 & 100.0\% & 49/109 \\
Q-learning ($\rho=0.7, \epsilon=0$) & 1.10783 $\pm$ 0.0419 & 326296.3000 $\pm$ 9296.8325 & 0.0\% & 0.0115 $\pm$ 0.0060 & 3333.8000 $\pm$ 1718.8984 & 0.3924 $\pm$ 0.1783 & 116500.0000 $\pm$ 52913.6088 & 100.0\% & 52/130 \\
Q-learning ($\rho=0.05, \epsilon=0.25$) & 2.82281 $\pm$ 0.0287 & 850959.0000 $\pm$ 3373.9975 & 0.0\% & 0.0077 $\pm$ 0.0055 & 2375.4000 $\pm$ 1552.5524 & nan $\pm$ nan & nan $\pm$ nan & 0.0\% & 54/170 \\
\hline
\end{tabular}}
\end{table}

\begin{table}[!h]
\caption{Performance of DP methods and  \ql{} as we change the learning rate $\rho$ in stochastic problem with $\gamma = 0.7$ (Problem 1).}
\label{tab:prob1stoch0.7}
\resizebox{\textwidth}{!}{
\begin{tabular}{lccccccccc}
\hline
\makecell{Algorithm (Problem 1) \\ ($\gamma$ = 0.7)} & \makecell{Run time \\ (mean $\pm$ std)}& \makecell{\#actions \\ (mean $\pm$ std)} & Convergence & \makecell{Discover goal \\ time \\ (mean $\pm$ std)} & \makecell{Discover goal \\ \#actions \\ (mean $\pm$ std) }& \makecell{Optimal Initial 
\\ Cost2Go Time \\ (mean $\pm$ std)} & \makecell{Optimal Initial \\ Cost2Go \#actions \\ (mean $\pm$ std)}& \makecell{Optimal Initial \\ Cost2Go \\ Convergence}& \makecell{Shortest/Longest \\ Path} \\
\hline
Q-learning ($\rho=0.05, \epsilon=0$) & 1.93393 $\pm$ 0.0889 & 518018.8000 $\pm$ 1079.4022 & 0.0\% & 0.0062 $\pm$ 0.0027 & 1755.0000 $\pm$ 772.9291 & nan $\pm$ nan & nan $\pm$ nan & 0.0\% & 37/312 \\
Q-learning ($\rho=0.1, \epsilon=0$) & 1.22601 $\pm$ 0.0191 & 346009.9000 $\pm$ 1428.1871 & 0.0\% & 0.0086 $\pm$ 0.0030 & 2369.8000 $\pm$ 840.9631 & 1.2029 $\pm$ 0.0272 & 340500.0000 $\pm$ 4193.2485 & 60.0\% & 35/92 \\
Q-learning ($\rho=0.3, \epsilon=0$) & 0.75712 $\pm$ 0.0197 & 211921.6000 $\pm$ 1266.0164 & 0.0\% & 0.0059 $\pm$ 0.0026 & 1675.6000 $\pm$ 592.5429 & 0.4032 $\pm$ 0.0182 & 113700.0000 $\pm$ 3848.3763 & 100.0\% & 36/49 \\
Q-learning ($\rho=0.5, \epsilon=0$) & 0.65779 $\pm$ 0.0115 & 186711.0000 $\pm$ 1534.7377 & 0.0\% & 0.0084 $\pm$ 0.0036 & 2143.8000 $\pm$ 1009.8904 & 0.2517 $\pm$ 0.0298 & 70900.0000 $\pm$ 8129.5756 & 100.0\% & 41/56 \\
Q-learning ($\rho=0.7, \epsilon=0$) & 0.63893 $\pm$ 0.0130 & 183386.4000 $\pm$ 4281.1417 & 0.0\% & 0.0061 $\pm$ 0.0021 & 1684.2000 $\pm$ 597.2734 & 0.2239 $\pm$ 0.0784 & 64200.0000 $\pm$ 21469.9790 & 100.0\% & 35/70 \\
Q-learning ($\rho=0.05, \epsilon=0.25$) & 1.84066 $\pm$ 0.0490 & 532172.3000 $\pm$ 1821.7158 & 0.0\% & 0.0074 $\pm$ 0.0034 & 2220.0000 $\pm$ 1046.5392 & nan $\pm$ nan & nan $\pm$ nan & 0.0\% & 41/312 \\
\hline
\end{tabular}%
}
\end{table}

\begin{table}[!h]
\caption{Performance of DP methods and \ql{} as we change the learning rate $\rho$ across Problems 1, 10, 11 with $\gamma = 0.5$ and $\epsilon$-decay.}
\label{tab:stoch0.5_eps_decay_all_problems}
\resizebox{\textwidth}{!}{
%
}
\end{table}

\begin{table}[!h]
\caption{Performance of DP methods and \ql{} as we change the learning rate $\rho$ across Problems 1, 10, 11 with $\gamma = 0.7$ and $\epsilon$-decay.}
\label{tab:stoch0.7_eps_decay_all_problems}
\resizebox{\textwidth}{!}{
%
}
\end{table}

\begin{table}[!h]
\caption{Performance of DP methods and \ql{} as we change the learning rate $\rho$ across Problems 1, 10, 11 with $\gamma = 0.9$ and $\epsilon$-decay.}
\label{tab:stoch0.9_eps_decay_all_problems}
\resizebox{\textwidth}{!}{
%
}
\end{table}

\begin{table}[!h]
\caption{Comparing performance of \ql{} as we change the learning rate $\rho$ in stochastic problem with $\gamma = 0.99$ (Problem 0). Two versions of value iteration are also appended for contrast}
\label{tab:prob0stoch0.99}
\resizebox{\textwidth}{!}{
%
}
\end{table}

\begin{table}[!h]
\caption{Comparing performance of \ql{} as we change the learning rate $\rho$ in stochastic problem with $\gamma = 0.999$ (Problem 0). Two versions of value iteration are also appended for contrast}
\label{tab:prob0stoch0.999}
\resizebox{\textwidth}{!}{
%
}
\end{table}

\begin{table}[!h]
\caption{Comparing performance of \ql{} as we change the learning rate $\rho$ in stochastic problem with $\gamma = 0.99$ (Problem 2). Two versions of value iteration are also appended for contrast}
\label{tab:prob2stoch0.99}
\resizebox{\textwidth}{!}{
%
}
\end{table}
\end{document}